\documentclass[sigconf,authorversion]{acmart}
\usepackage{hyperref}
\usepackage{url}

\usepackage[utf8]{inputenc} 
\usepackage[T1]{fontenc}    
\usepackage{hyperref}       
\usepackage{url}            
\usepackage{booktabs}       
\usepackage{amsfonts}       
\usepackage{nicefrac}       
\usepackage{microtype}      
\usepackage{xcolor}         
\usepackage{subfigure}
\usepackage{graphicx}
\usepackage{mathtools}
\usepackage{amsmath}
\usepackage{enumitem}
\usepackage{amsthm}
\usepackage{xspace}
\usepackage{algorithm}
\usepackage{algorithmic}
\DeclareMathAlphabet{\mathcal}{OMS}{cmsy}{m}{n}
\usepackage{bbm}
\usepackage{multirow}
\usepackage{wrapfig}
\usepackage{booktabs}

\newcommand{\aka}{\emph{a.k.a.,}\xspace}

\newcommand{\wrt}{\emph{w.r.t.}\xspace}
\newcommand{\ie}{\emph{i.e.,}\xspace}
\newcommand{\ignore}[1]{}

\newcommand{\model}{DR-GST}
\newcommand{\ds}{distribution shift}

\AtBeginDocument{%
  \providecommand\BibTeX{{%
    \normalfont B\kern-0.5em{\scshape i\kern-0.25em b}\kern-0.8em\TeX}}}

\setcopyright{acmcopyright}
\copyrightyear{2022}
\acmYear{2022}
\setcopyright{acmcopyright}\acmConference[WWW '22]{Proceedings of the ACM Web
Conference 2022}{April 25--29, 2022}{Virtual Event, Lyon, France}
\acmBooktitle{Proceedings of the ACM Web Conference 2022 (WWW '22), April
25--29, 2022, Virtual Event, Lyon, France}
\acmPrice{15.00}
\acmDOI{10.1145/3485447.3512172}
\acmISBN{978-1-4503-9096-5/22/04}

\begin{document}

\title{Confidence May Cheat: Self-Training on Graph Neural Networks under Distribution Shift}

\author{Hongrui Liu}
\email{liuhongrui@bupt.edu.cn}
\authornote{Work done during internship at Ant Group.}
\affiliation{%
  \institution{Beijing University of Posts and Telecommunications}}

\author{Binbin Hu}
\email{bin.hbb@antfin.com}
\affiliation{%
  \institution{Ant Group}}

\author{Xiao Wang}
\email{xiaowang@bupt.edu.cn}
\affiliation{%
  \institution{Beijing University of Posts and Telecommunications}
  \institution{Peng Cheng Laboratory}}

\author{Chuan Shi}
\authornote{Corresponding author}
\email{shichuan@bupt.edu.cn}
\affiliation{%
  \institution{Beijing University of Posts and Telecommunications}
  \institution{Peng Cheng Laboratory}}

\author{Zhiqiang Zhang}
\email{lingyao.zzq@antfin.com}
\affiliation{%
  \institution{Ant Group}}

\author{Jun Zhou}
\email{jun.zhoujun@antfin.com}
\affiliation{%
  \institution{Ant Group}}

\begin{abstract}
  Graph Convolutional Networks (GCNs) have recently attracted vast interest and achieved state-of-the-art performance on graphs, but its success could typically hinge on careful training with amounts of expensive and time-consuming labeled data. To alleviate labeled data scarcity, self-training methods have been widely adopted on graphs by labeling high-confidence unlabeled nodes and then adding them to the training step. In this line, we empirically make a thorough study for current self-training methods on graphs. Surprisingly, we find that high-confidence unlabeled nodes are not always useful, and even introduce the {\ds} issue between the original labeled dataset and the augmented dataset by self-training, severely hindering the capability of self-training on graphs.
  To this end, in this paper, we propose a novel \textbf{D}istribution \textbf{R}ecovered \textbf{G}raph \textbf{S}elf-\textbf{T}raining framework (\model), which could recover the distribution of the original labeled dataset.
  Specifically, we first prove the equality of loss function in self-training framework under the distribution shift case and the population distribution if each pseudo-labeled node is weighted by a proper coefficient. Considering the intractability of the coefficient, we then propose to replace the coefficient with the information gain after observing the same changing trend between them, where information gain is respectively estimated via both dropout variational inference and dropedge variational inference in DR-GST. However, such a weighted loss function will enlarge the impact of incorrect pseudo labels. As a result, we apply the loss correction method to improve the quality of pseudo labels. 
  Both our theoretical analysis and extensive experiments on five benchmark datasets demonstrate the effectiveness of the proposed \model, as well as each well-designed component in \model.

\end{abstract}





\maketitle

\section{Introduction}
\label{sec:intro}
Graphs are ubiquitous across many real-world applications, ranging from citation and social
network analysis to protein interface and chemical bond prediction. 
With the surge of demands,  Graph Convolution Network (GCN) and its variants \cite{gcn, gat, ppnp, gin, sgc} (abbreviated as GCNs) have recently attracted vast interest and achieved state-of-the-art performance in various tasks on graphs, most notably semi-supervised node classification. 
Nevertheless, its success could typically hinge on careful training with large amounts of labeled data, which is expensive and time-consuming to be obtained \cite{m3s}. 
Empirically, the performance of GCNs will rapidly decline with the decrease of labeled data \citep{abn}.

As one of the promising approaches, self-training~\cite{self-training,pseudo-labeling} aims at addressing labeled data scarcity by making full use of abundant unlabeled data in addition to task-specific labeled data.   
Given an arbitrary model trained on the original labeled data as the \emph{teacher} model, the key idea of self-training is to pseudo-label high-confidence unlabeled samples to augment the above labeled data, and a \emph{student} model is trained with augmented data to replace the \emph{teacher} model. 
Such an iteration learning is repeated until convergence~\footnote{The \emph{teacher-student} term is commonly adopted in current self-training studies~\cite{uncertainty, self-training-cv1, self-neural}, and we just reuse it here for a clearer explanation. }. 
Analogously, self-training has great potential to facilitate advancing GCNs to exploiting unlabeled data ~\cite{deeper, m3s, abn}. 
Whereas, these studies only focus on the high-confidence nodes on account of the prefabricated assumption that the higher the confidence, the more accurate the prediction. Naturally, we are curious about such a fundamental question, ``\emph{Are all the unlabeled nodes pseudo-labeled with high confidence truly useful?}''



As a motivating example, we conduct an analysis experiment on a benchmark dataset Cora~\cite{cora} to explore how much additional information these high-confidence nodes can bring to the model (denoted as information gain). More details can be seen in Section \ref{sec:analysis}. Surprisingly, our experimental results show a clear negative correlation between the confidence and the information gain, implying that nodes pseudo-labeled by existing graph self-training methods with high confidence may be low-information-gain and useless. To further understand the underlying reason, we illustrate the distribution of unlabeled nodes and find these high-confidence (or low-information-gain) nodes are far from the decision boundary, which implies that they potentially guide the model to perform worthless optimization for a more crisp decision boundary.
Existing graph self-training methods which focus on high-confidence nodes are ``cheated'' by confidence in this way. 
In light of the above observations, we further investigate into what will happen when self-training is cheated by confidence.
We discover that during the optimization procedure dominated by \emph{easy nodes} (\ie nodes with low information gain), the \emph{Distribution Shift} phenomenon between the original and augmented dataset gradually appears. This is because more and more easy nodes selected by high confidence are added to the original labeled dataset, leading to the distribution gradually shifting to the augmented dataset and overmuch attention paid on such easy nodes as a result. Not surprisingly, this issue will severely threaten the capacity of self-training on graphs, since the distribution of the augmented dataset is different from the population distribution, resulting in a terrible generalization during evaluation. Alleviating {\ds} from self-training on graphs is in urgent demand, which is unexplored in existing studies.

In this paper, we propose an information gain weighted self-training framework {\model}  which could recover the distribution of original labeled dataset.
Specifically, we first prove that the loss function of the self-training framework under the {\ds} case is equal to that under the population distribution if we could weight each pseudo-labeled node with a proper coefficient. 
But the coefficient is generally intractable in practice.
Then we discover the same changing trend between the coefficient and information gain, and propose to replace the coefficient with information gain, where the information gain can be estimated via both dropout variational inference and dropedge variational inference. 
Consequently, we can recover the shifted distribution with the newly proposed information gain weighted loss function. 
Such a loss function forces the model to pay more attention to \emph{hard nodes}, \ie nodes with high information gain, but will enlarge the impact of incorrect pseudo labels. Therefore, we apply loss correction  \cite{loss_correction1,loss_correction2,loss_correction3} to self-training to correct the prediction of the student model, so that the impact of incorrect pseudo labels from the teacher model can be alleviated in this way. 
Finally, we conduct a theoretical analysis of self-training on graphs, and the conclusion shows both {\ds} and incorrect pseudo labels will severely hinder its capability, which is consistent with our designs.

In summary, the main contributions are highlighted as follows: 
\begin{itemize}[leftmargin=*]
\item We make a thorough study on graph self-training, and find two phenomena below: 1) pseudo-labeled high-confidence nodes may cheat. 2) {\ds} between the original labeled dataset and the augmented dataset. Both of them severely hinder the capability of self-training on graphs.


\item We propose a novel graph self-training framework {\model} that not only addresses the {\ds} issue from the view of information gain, but also is equipped with the creative loss correction strategy for improving qualities of pseudo labels.

\item We theoretically analyze the rationality of the whole {\model} framework and extensive experimental results on five benchmark datasets demonstrates that {\model} consistently and significantly outperforms various state-of-arts.
\end{itemize}

\section{Preliminary}
\label{sec:pre}
Let $\mathcal{G} = \left(\mathcal{V,E}, \mathbf{X}\right)$ be a graph with the adjacent matrix 
$\mathbf{A}\in\mathbb{R}^{|\mathcal{V}|\times |\mathcal{V}|}$, where $\mathcal{V}$ and  $\mathcal{E}$ are respectively the set of nodes and edges, and  $\mathbf{X}=[\mathbf{x}_1,\mathbf{x}_2,\cdots,\mathbf{x}_{|\mathcal{V}|}] \in\mathbb{R}^{|\mathcal{V}|\times D_v}$ is the $D_v$-dimensional feature matrix for nodes.
In the common semi-supervised node classification setting, we only have access to a small amounts of labeled nodes 
$\mathcal{V}_L$ with their labels $\mathcal{Y}_L$ along with a larger amounts of unlabeled nodes 
$\mathcal{V}_U$, where $|\mathcal{V}_L| \ll |\mathcal{V}_U|$.

\textbf{Self-training}
Generally, self-training methods on graphs firstly train a vanilla GCN as the base \emph{teacher} model 
$f_\theta\left(\mathbf{X},\mathbf{A}\right)$ with ground-truth labels $\mathcal{Y}_L$, where $\theta$ is the model parameter set. 
We could obtain the probability vector for each node $v_i\in\mathcal{V}$ as: $\mathbf{p}\left(y_i|\mathbf{x}_i,\mathbf{A};\theta\right)=f_\theta\left(\mathbf{x}_i,\mathbf{A}\right)$. For convenience, we abbreviate it to $\mathbf{p}_i$ and denote the \emph{j}-th element of $\mathbf{p}_i$ by $p_{i,j}$. Next, the teacher model pseudo-labels a subset $\mathcal{S}_U\subset\mathcal{V}_U$ of unlabeled nodes with its prediction $\bar{y}_u=\arg\max_j p_{u,j}$ 
for each node $v_u\in\mathcal{S}_U$. The selection of $\mathcal{S}_U$ is based on the confidence score $r_i=\max_{j} p_{i,j}$, \ie only nodes with $r_i$ higher than a threshold or top-\emph{k} high-confidence nodes are added to the labeled dataset. 
Then the augmented dataset $\mathcal{V}_L\cup\mathcal{S}_U$ is used to train a \emph{student} model $f_{\bar{\theta}}$ 
with the following objective function.

\begin{equation}
  \label{eq:objective_func}
   \begin{aligned}
      \min_{\bar{\theta}\in\Theta}\mathcal{L}\left(\mathbf{A,X},\mathcal{Y}_L\right)&=\min_{\bar{\theta}\in\Theta}\mathbb{E}_{v_i\in\mathcal{V}_L,y_i\in\mathcal{Y}_L}l\left(y_i,\mathbf{p}_i\right) \\
      &+\lambda\mathbb{E}_{v_u\in\mathcal{S}_U,\mathcal{S}_U\subset\mathcal{V}_U}\mathbb{E}_{\bar{y}_u\thicksim\mathbf{p}\left(y_u|\mathbf{x}_u,A;\theta\right)}l\left(\bar{y}_u,\mathbf{p}_u\right),
   \end{aligned}
\end{equation}
where $l\left(y_i,\mathbf{p}_i\right) = -\log p_{i,y_i}$ is the multi-class cross entropy loss  and we fix $\lambda=1$ in this paper. 
Finally we replace the teacher model with the student model and iterate the above procedure until convergence.

\textbf{Information Gain}
As can be seen in Eq. \ref{eq:objective_func}, self-training on graphs will exploit the unlabeled data to train the whole model. Here, we aim to measure how an unlabeled node contributes to the model optimization in a principled way, \ie information gain. 
Information gain usually measures the reduction in information given a random variable, where information is generally calculated by the Shannon's entropy \cite{entropy}.
We utilize the information gain here to seek the node $v_u$ which owns the most information about parameters $\theta$ of model posterior and could reduce the number of possible parameter hypotheses maximally fast. We refer to this type of information gain as information gain about model parameters \cite{uncertainty}.
Formally, given a node $v_u$, the information gain about model parameters is defined as $\mathbb{B}_u$, which could be calculated as follows:
\begin{equation}
\label{eq:information_gain_first}
\begin{aligned}
  \mathbb{B}_u(y_u,\theta|\mathbf{x}_u,\mathbf{A},\mathcal{G}) = \mathbb{H}[\mathbb{E}_{P(\theta|\mathcal{G})}[y_u|\mathbf{x}_u,\mathbf{A};\theta]] \\- 
  \mathbb{E}_{P(\theta|\mathcal{G})}[\mathbb{H}[y_u|\mathbf{x}_u,\mathbf{A};\theta]],
\end{aligned}
\end{equation}
where $\mathbb{H}(\cdot)$ denotes the Shannon's entropy and $P(\theta|\mathcal{G})$ is the distribution of model posterior.  
The first term measures the information of the model parameters under posterior, while the second term captures the information of model parameters given an additional node $v_u$. 
Obviously, by calculating the difference between the two terms above, $\mathbb{B}_u$ can measure how much information $v_u$ can bring to learn the model parameters $\theta$.

\section{Empirical Analysis}
\label{sec:analysis}
In this section, we conduct a series of empirical analysis to examine whether current graph self-training approaches adopt a principled way to leverage unlabeled data for semi-supervised node classification. 

\textbf{Empirical Analysis of Confidence}
\label{subsec:empirical_conf}
To better understand the capacity of high-confidence nodes in current self-training approaches, we aim to closely examine that how much additional information these nodes can bring to the model based on information gain.   
We first visualize the relationship between confidence and information gain in Fig. \ref{fig:conf_un_cora}, where the x-axis is the confidence while the y-axis is the information gain, and the blue and orange dots respectively represent nodes with correct and incorrect predictions. From Fig. \ref{fig:conf_un_cora} we can observe a negative correlation, implying that existing graph self-training methods only focus on \emph{easy} nodes (nodes with low information gain) and confidence may be cheating as a result. 
Essentially, such a ``cheating" phenomenon lies in the worthless optimization for a more crisp decision boundary. Specifically, as shown in Fig. \ref{fig:distribution_cora}, on the Cora dataset, we visualize the node embeddings on the last layer of the standard GCN before \emph{softmax} using \emph{t}-SNE \citep{tsne} algorithm,  where a darker dot represents a node with lower information gain. From the plots, we find that most of 
easy nodes (\ie low information gain) are far from the decision boundary. Whereas, these nodes are always emphasized by current self-training methods on graphs~\cite{deeper, m3s, abn} by force of high confidence. That is, these methods are ``cheated'' by confidence in this way. 

\begin{figure}
	\centering
	\subfigure[]{\label{fig:conf_un_cora}\includegraphics[width=0.45\columnwidth]{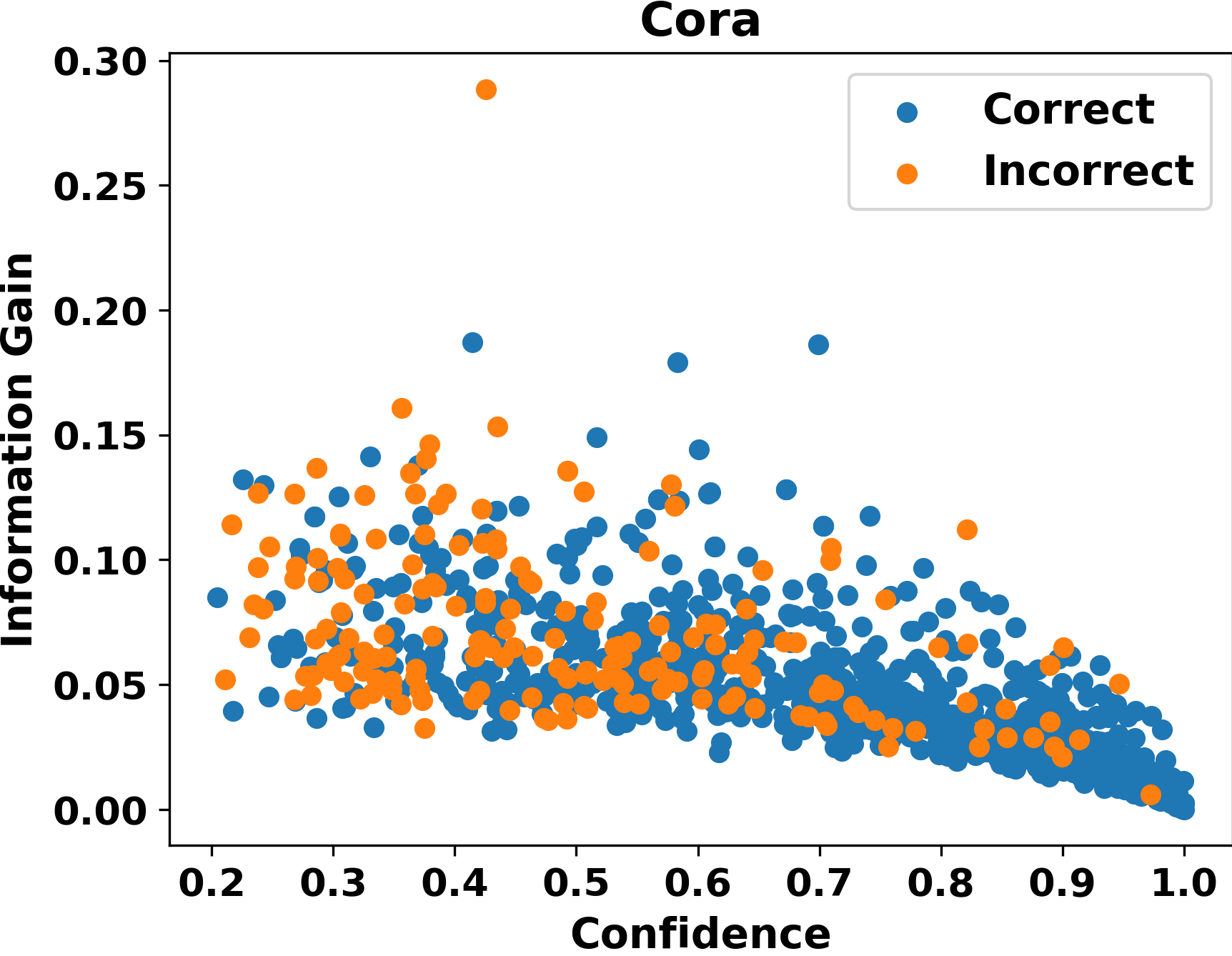}}
	\subfigure[]{\label{fig:distribution_cora}\includegraphics[width=0.45\columnwidth]{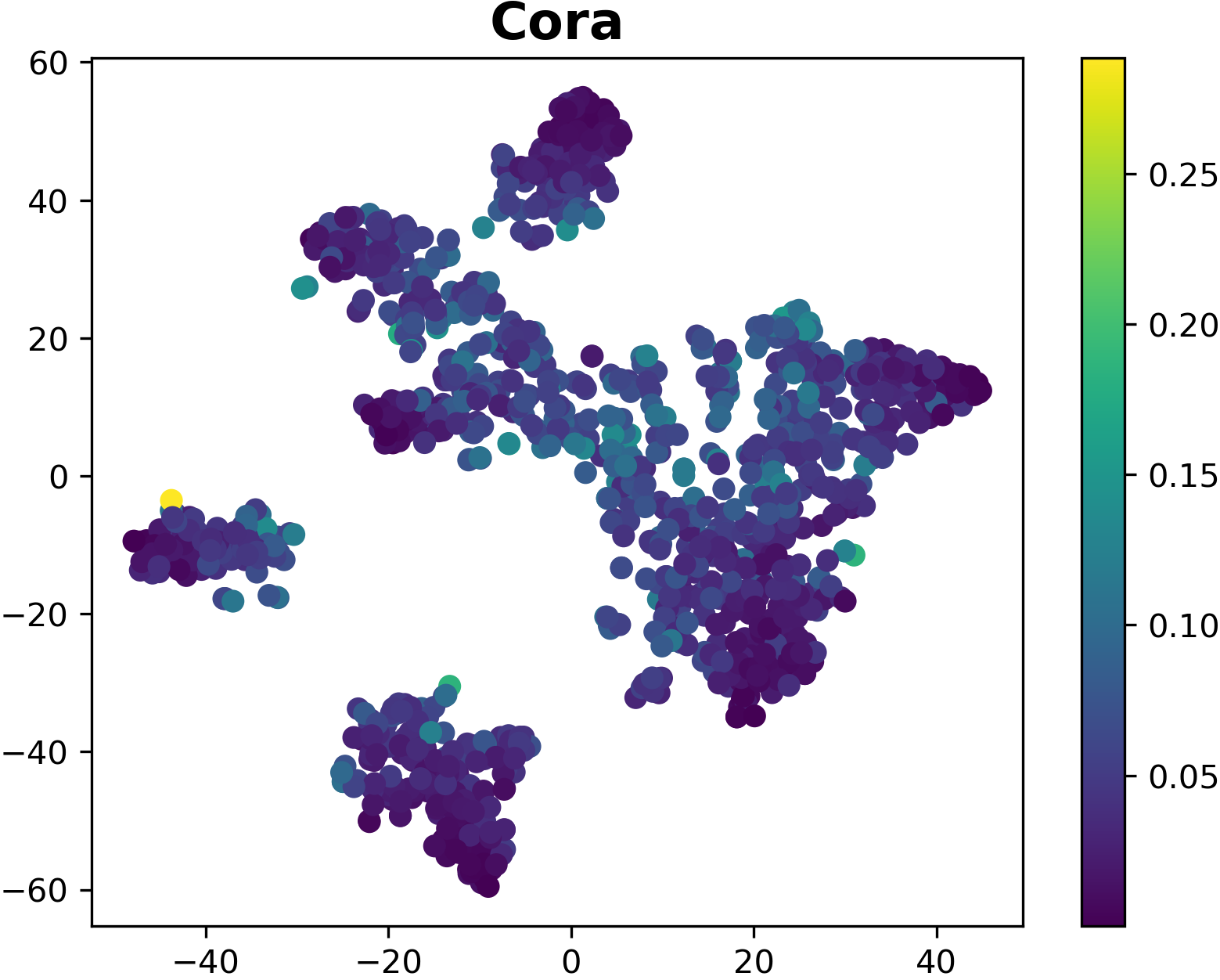}}
  \caption{(a): Relationship between confidence and information gain on Cora. (b): Visualization of embeddings on Cora}
\end{figure}



\begin{figure}
	\centering
	\subfigure[$P_{pop}$]{\label{fig:illustration_before}\includegraphics[width=0.32\columnwidth]{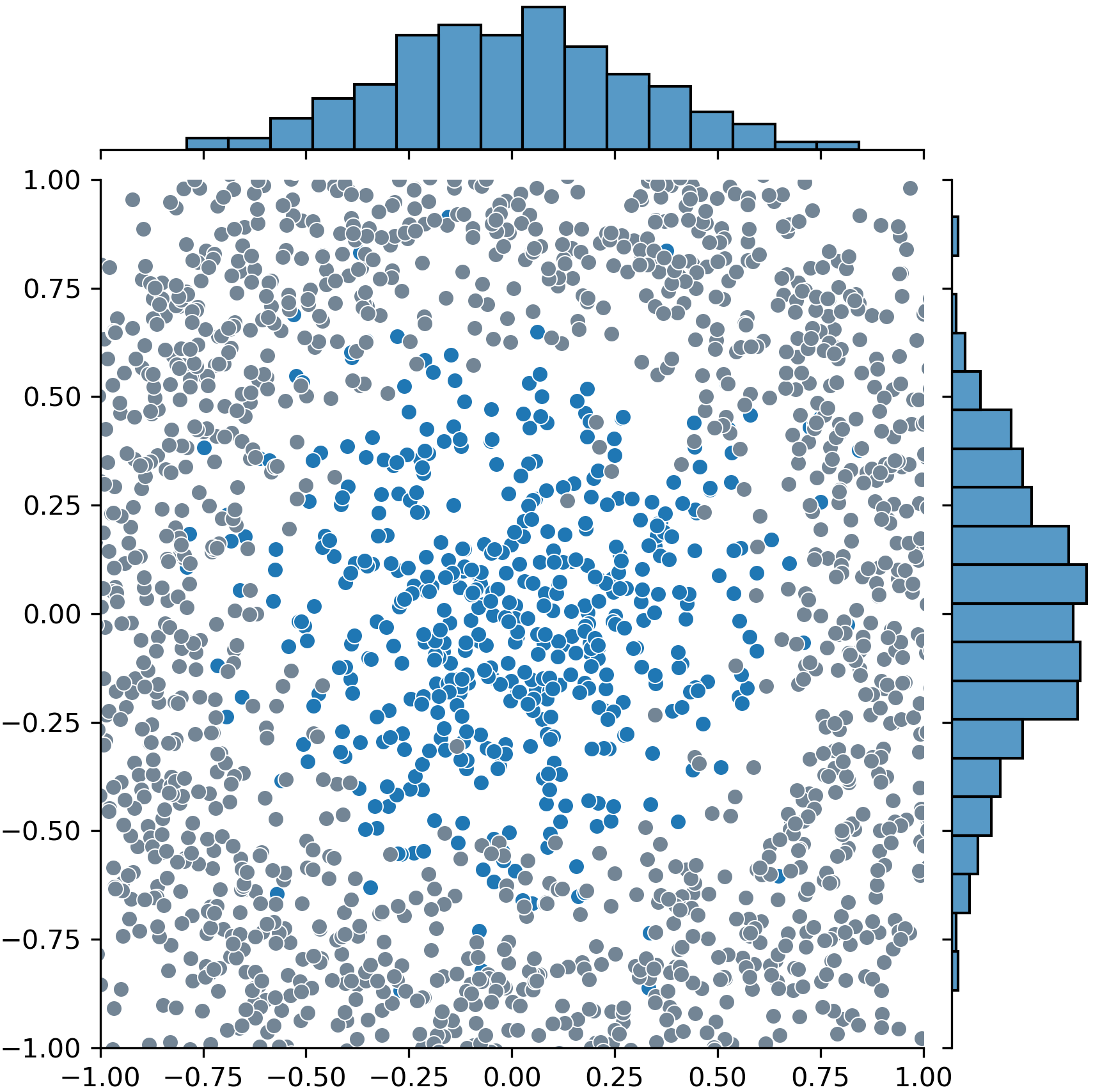}}
      \subfigure[$P_{st}$]{\label{fig:illustration_after}\includegraphics[width=0.32\columnwidth]{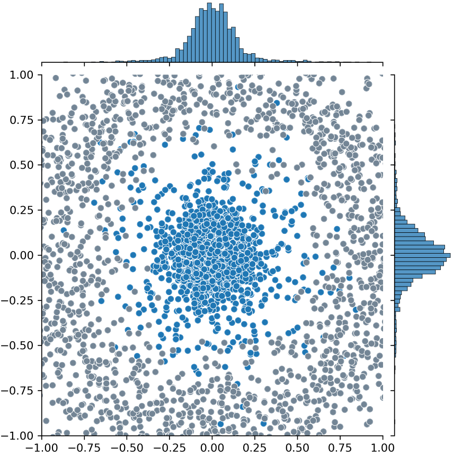}}
      \subfigure[ratio of $P_{pop}$ and $P_{st}$]{\label{fig:ratio}\includegraphics[width=0.32\columnwidth]{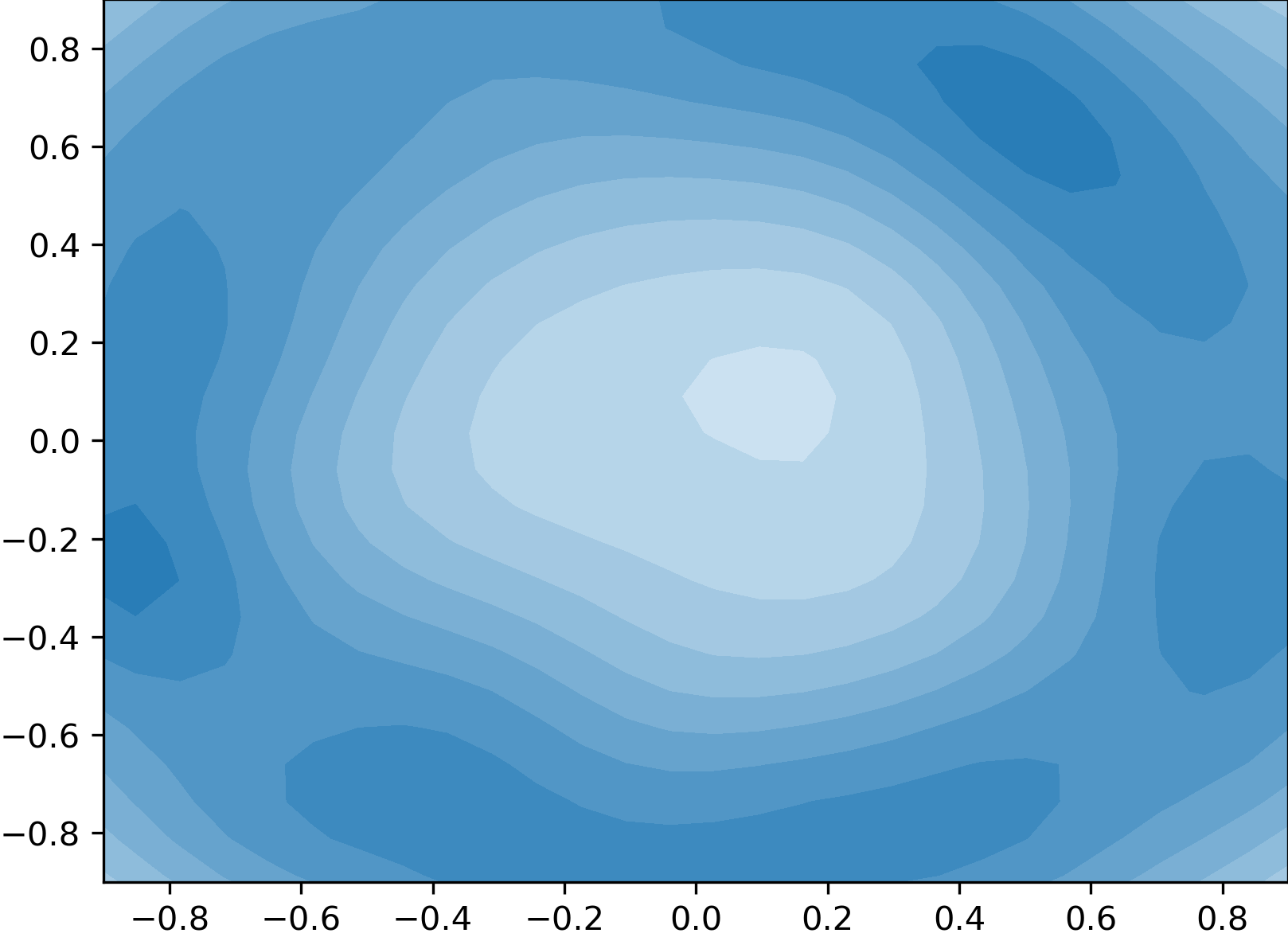}}
  \caption{Visualization of labeled nodes under the ideal condition. ($P_{pop}$: distribution before self-training, $P_{st}$: distribution after self-training)}
	\label{fig:illustration}
\end{figure}

\textbf{Empirical Analysis of Distribution Shift}
Furthermore, we investigate what will happen when self-training has been cheated by confidence. As an illustrative example, we randomly generate 500 nodes (blue) following two-dimensional Gaussion distribution $\mathcal{N}(0,0,0.3,0.3,0)$ to represent labeled nodes in one class, and another 4000 nodes (grey) following the distribution of concentric circles \cite{confidence_score} to represent labeled nodes belonging to other classes, as shown in Fig. \ref{fig:illustration_before}. Furthermore, following the common self-training setting, a large amount of unlabeled nodes still exists in the dataset, but for clarity, we omit them in the figure.
In line with the core idea of current self-training methods, for the ``blue'' class,  unlabeled nodes around the center are pseudo-labeled for self-training since these nodes 
have high confidence (\aka far from the decision boundary). During iteration, as shown in Fig~\ref{fig:illustration_after}, the data distribution will become more and more sharpen since nodes far from the decision boundary are paid disproportionate attention and thus the unsatisfying \emph{Distribution Shift} phenomenon between the original and augmented dataset indeed appears.
\section{The {\model} Framework}
\label{sec:framework}
In this section,  we elaborate the proposed {\model}, a novel self-training framework aiming at recovering the shifted distribution. 


\subsection{Information Gain Weighted Loss Function Towards Distribution Shift}
\label{subsec:weighted}
We start with the formulation of the self-training task by analyzing the corresponding loss functions.
Specifically, assuming that the original labeled dataset follows the population distribution $P_{pop}$, given a classifier $f_\theta$ parameterized by $\theta$, the best parameter set $\theta$ could be obtained via minimizing the following loss function:
\begin{equation}
  \label{eq:loss_pop}
    \mathcal{L}_{pop} = \mathbb{E}_{(v_i,y_i)\thicksim P_{pop}(\mathcal{V,Y})}l(y_i,\mathbf{p}_i).
\end{equation}
Similarly, under the {\ds} case caused by self-training, the loss function can be represented as
\begin{equation}
  \label{eq:loss_st}
    \begin{aligned}
        \mathcal{L}_{st} &= \frac{|\mathcal{V}_L|}{|\mathcal{V}_L\cup\mathcal{S}_U|} \mathbb{E}_{(v_i,y_i)\thicksim P_{pop}(\mathcal{V,Y})}l(y_i,\mathbf{p}_i) \\
        &+ \frac{|\mathcal{S}_U|}{|\mathcal{V}_L\cup\mathcal{S}_U|}\mathbb{E}_{(v_u,y_u)\thicksim P_{st}(\mathcal{V,Y})}l(\bar{y}_u,\mathbf{p}_u),
    \end{aligned}
\end{equation}
where $P_{st}$ represents the shifted distribution of the augmented dataset.

Generally, the {\ds} could lead to a terrible generalization during evaluation, and thus severely threaten the capacity of graph self-training. Therefore, It is ideal to optimize $f_\theta$ with the loss function $\mathcal{L}_{pop}$ under the population distribution rather than  $\mathcal{L}_{st}$ under the {\ds} case. However, only $\mathcal{L}_{st}$ is available in practice. To close the gap, we show the following theorem.
\begin{theorem}
\label{theory:loss_function}
    Given $\mathcal{L}_{pop}$ and $\mathcal{L}_{st}$ defined in Eq. \ref{eq:loss_pop} and Eq. \ref{eq:loss_st}, assuming that $\bar{y}_u=y_u$ for each pseudo-labeled node $v_u\in\mathcal{S}_U$, 
    then $\mathcal{L}_{st}=\mathcal{L}_{pop}$ holds true if $\mathcal{L}_{st}$ can be written with an additional weight coefficient $\gamma_u=\frac{P_{pop}(v_u,y_u)}{P_{st}(v_u,y_u)}$ as follows:
    \begin{equation}
    \label{eq:loss_equality}
        \begin{aligned}
            \mathcal{L}_{st} &= \frac{|\mathcal{S}_U|}{|\mathcal{V}_L\cup\mathcal{S}_U|}\mathbb{E}_{(v_u,y_u)\thicksim P_{st}(\mathcal{V,Y})}\gamma_u l(\bar{y}_u,\mathbf{p}_u)\\
            &+\frac{|\mathcal{V}_L|}{|\mathcal{V}_L\cup\mathcal{S}_U|}\mathbb{E}_{(v_i,y_i)\thicksim P_{pop}(\mathcal{V,Y})}l(y_i,\mathbf{p}_i),
        \end{aligned}
    \end{equation}
\end{theorem}
\begin{proof}
  Please refer to \ref{appendix:proof_loss_function}.
\end{proof}

Based on Theorem \ref{theory:loss_function}, we can find that our desired $\mathcal{L}_{pop}$ can be written as the available $\mathcal{L}_{st}$ only if a coefficient $\gamma_u$ is added to $\mathcal{L}_{st}$. In other words, the {\ds} issue could be addressed by optimizing $f_\theta$ with available $\mathcal{L}_{st}$ weighted by $\gamma_u$ (in Eq. \ref{eq:loss_equality}).
However, it should be noted that the population distribution $P_{pop}$ in $\mathcal{L}_{st}$ is generally intractable, which means that $\gamma_u$ cannot be accurately calculated. 


To this end, we propose to build the bridge between $\gamma_u$ and the information gain, which is motivated as follows.
Recalling the data distributions shown in Fig. \ref{fig:illustration_before} and Fig. \ref{fig:illustration_after}, we could formally represent the former as $P_{pop}$ and the latter as $P_{st}$. We visualize the desired weight coefficient $\gamma_u = \frac{P_{pop}(v_u,y_u)}{P_{st}(v_u,y_u)}$ for each pseudo-labeled node $v_u$ in Fig. \ref{fig:ratio} for better understanding its changing trend, where the darker area means the larger $\gamma_u$.  Obviously, we observe that $\gamma_u$ becomes smaller when getting closer to the center area (\aka far away from the decision boundary), which is consistent with the change trend of the information gain. This finding inspires us to adopt the information gain to approximate $\gamma_u$.


\subsection{Information Gain Estimation on Graphs}
\label{subsec:information_gain}
Next, we elaborate the estimation of the information gain for each node $v_u$ in graph. As mentioned in Eq. \ref{eq:information_gain_first}, the distribution of model posterior $P(\theta|\mathcal{G})$ is desired for calculating information gain,
but it is intractable in practice, and always computationally expensive for traditional bayesian neural networks~\cite{bayesian1,bayesian2,bayesian3}.
Instead, we could shift attention towards dropout~\cite{dropout} and dropedge~\cite{dropedge}, a type of regularization technique for preventing over-fitting and over-smoothing in GCNs, which could be both interpreted as an approximation of $P(\theta|\mathcal{G})$~\cite{gdc}. Consequently, we propose to estimate  the information gain assisted with dropout and dropedge ({\aka} \emph{dropout} and \emph{dropedge variational inference}), which takes into account both features and the network topology in our unified framework {\model}. For distinction, we refer to {\model} with dropout variational inference as {\model}$_{do}$ and that with dropedge variational inference as {\model}$_{de}$.

\subsubsection{Dropout Variational Inference}
Specifically, given a $L$-layer GCN model $f_\theta$, its $l$-th layer output $\mathbf{H}^{(l)} \in \mathbb{R}^{|\mathcal{V}|\times D_l}$ 
can be obtained by 
\begin{equation}
  \mathbf{H}^{(l)}=\sigma(\mathfrak{N}(\mathbf{A})\mathbf{H}^{(l-1)}\mathbf{W}^{(l-1)}),
\end{equation} 
where $\mathfrak{N}(\cdot)$ represents the normalizing operator, 
$\mathbf{W}^{(l-1)}\in\mathbb{R}^{D_{l-1}\times D_l}$ is the (\emph{l}-1)-th layer weight matrix, $\sigma(\cdot)$ is the activation function and 
$\mathbf{H}^{(1)}=\mathbf{X}\in\mathbb{R}^{|\mathcal{V}|\times D_v}$, $\theta=\{\mathbf{W}^{(l)}\}_{l=1}^L$. Dropout randomly masks features of nodes in the graph through drawing from an independent Bernoulli random variable.
Formally, the $l$-th layer output of $f_\theta$ with dropout can be written as:
\begin{equation}
  \label{eq:dropout}
  \mathbf{H}^{(l)}=\sigma(\mathfrak{N}(\mathbf{A})(\mathbf{H}^{(l-1)}\odot\mathbf{Z}^{(l-1)})\mathbf{W}^{(l-1)}),
\end{equation}
where each element of $\mathbf{Z}^{(l)}\in \{0,1\}^{D_{l-1}\times D_{l-1}}$ is a sample of Bernoulli random variable, representing whether or not the corresponding feature in $\mathbf{H}^{(l-1)}$ is set to zero.

Such Bernoulli random sampling on features can also be treated as a sample from $P(\theta|\mathcal{G})$ \cite{dropout}, thus we can perform $T$-times Monte-Carlo sampling (referred to Monte-Carlo dropout, \emph{MC-dropout}) during \emph{inference} to estimate $P(\theta|\mathcal{G})$. At each time $t$, a probability vector $\tilde{\mathbf{p}}_u^t = \tilde{\mathbf{p}}^t(y_u|\mathbf{x}_u,\mathbf{A};\tilde{\theta}_t)$ 
can be obtained by performing forward pass under such a sample weight $\tilde{\theta}_t$, \ie $\tilde{\mathbf{p}}_u^t=f_{\tilde{\theta}^t}(\mathbf{x}_u,\mathbf{A})$. 

However, from the perspective of the computational overhead and practical performance, we only conduct dropout on the last layer during MC-dropout.
In other words, the probability vector $\tilde{\mathbf{p}}_u^t\in\tilde{\mathbf{P}}^t=f_{\tilde{\theta}^t}({\mathbf{X},\mathbf{A}})$ at each time $t$ can be obtained by:
\begin{equation}
\label{eq:dropout_p}
  \tilde{\mathbf{P}}^t=\sigma(\mathfrak{N}(\mathbf{A})(\mathbf{Z}^{(t)}\odot\sigma(\mathfrak{N}(\mathbf{A})\cdots\sigma(\mathfrak{N}(\mathbf{A})\mathbf{XW}^{(1)})\cdots)\mathbf{W}^{(l-1)}))\mathbf{W}^{(l)})
\end{equation}



\subsubsection{Dropedge Variational Inference}
The dropedge variational inference takes a similar way with dropout variation inference, but imposes the randomness on the network topology instead. 

Specifically, the $l$-th layer output of $f_\theta$ with dropedge can be written as:
\begin{equation}
  \label{eq:dropedge}
  \mathbf{H}^{(l)}=\sigma(\mathfrak{N}(\mathbf{A}\odot\mathbf{Z}^{(l-1)})\mathbf{H}^{(l-1)}\mathbf{W}^{(l-1)}),
\end{equation}
where each element of $\mathbf{Z}^{(l)}\in \{0,1\}^{|\mathcal{V}|\times |\mathcal{V}|}$ is also a sample of Bernoulli random variable, 
representing whether or not the corresponding edge in $\mathbf{A}$ is removed.

Similarly, we only conduct dropedge on the last layer and perform $T$-times Monte-Carlo sampling (referred to as Monte-Carlo dropedge) base on dropedge, where at each time \emph{t}, the probability vector 
$\tilde{\mathbf{p}}_u^t\in\tilde{\mathbf{P}}^t=f_{\tilde{\theta}^t}({\mathbf{X},\mathbf{A}})$ at each time $t$ is obtained by
\begin{equation}
\label{eq:dropedge_p}
  \tilde{\mathbf{P}}^t=\sigma(\mathfrak{N}(\mathbf{A}\odot\mathbf{Z}^{(t)})\sigma(\mathfrak{N}(\mathbf{A})\cdots\sigma(\mathfrak{N}(\mathbf{A})\mathbf{XW}^{(1)})\cdots)\mathbf{W}^{(l-1)})\mathbf{W}^{(l)}).
\end{equation}

\subsubsection{Information Gain Estimation}
With such probability vector $\tilde{\mathbf{p}}_u^t$ obtained by Eq. \ref{eq:dropout_p} or Eq. \ref{eq:dropedge_p}, we can calculate the prediction distribution $\mathbf{p}_u^\mathcal{G}$ by averaging all the $\tilde{\mathbf{p}}_u^t$: 
\begin{equation}
  \label{eq:predictive_distribution}
  \mathbf{p}_u^{\mathcal{G}}=\mathbf{p}(y_u|\mathbf{x}_u,\mathbf{A},\mathcal{G})=\frac{1}{T}\sum_{t=1}^T \tilde{\mathbf{p}}_u^t, 
  \tilde{\theta}_t\sim P(\theta|\mathcal{G}),
\end{equation}
and thus the information gain $\mathbb{B}_u$ can be calculated by: 
\begin{equation}
  \label{eq:information_gain}
  \mathbb{B}_u(y_u,\theta|\mathbf{x}_u,\mathbf{A},\mathcal{G})=-\sum_{d=1}^D p_{u,d}^\mathcal{G}\log p_{u,d}^\mathcal{G}+\frac{1}{T}\sum_{d=1}^D\sum_{t=1}^T \tilde{p}_{u,d}^t \log \tilde{p}_{u,d}^t.
\end{equation}

Finally, we weight the loss function with above information gain after normalization:


\begin{equation}
  \label{eq:final_loss}
  \begin{aligned}
  \mathcal{L}_{st}
  & = \frac{|\mathcal{S}_U|}{|\mathcal{V}_L\cup\mathcal{S}_U|}\mathbb{E}_{(v_u,y_u)\thicksim P_{st}(\mathcal{V,Y})}\bar{\mathbb{B}}_u l(\bar{y}_u,\mathbf{p}_u)\\
  &+\frac{|\mathcal{V}_L|}{|\mathcal{V}_L\cup\mathcal{S}_U|}\mathbb{E}_{(v_i,y_i) \thicksim P_{pop}(\mathcal{V,Y})}l(y_i,\mathbf{p}_i) \\
  & \text{where\ \ \ \ \ \ \ \ }\bar{\mathbb{B}}_u=\frac{\mathbb{B}_u}{\beta\cdot\frac{1}{|\mathcal{S}_U|}\sum_i\mathbb{B}_i}.
\end{aligned}
\end{equation}
Here, we can tune the balance coefficient $\beta$ to recover the population distribution (\ie $\mathcal{L}_{st} \approx \mathcal{L}_{pop}$) as much as possible.


\subsection{Improving Qualities of Pseudo Labels via Loss Correction}
\label{subsec:loss_correction}

Till now, we have addressed the {\ds} issue with an information gain weighted loss function, where more attentions are paid to nodes with high information gain rather than high confidence. Unfortunately, such a training pipeline still implies hidden risks.
Specifically, considering that pseudo labels of hard nodes are more likely to be incorrect as shown in Fig. \ref{fig:conf_un_cora} and our DR-GST focuses more on hard nodes, the impact of incorrect pseudo-labeled nodes will be enlarged and even mislead the learning of GCNs. Previous works generally filter out these low-quality nodes with collaborative scoring~\cite{deeper, m3s} or prefabricated assumption~\cite{abn} in a relatively coarse-grained manner, where abundant nodes with high information gain are discarded in advance.
Instead, motivated by studies on learning with noisy labels~\cite{loss_correction1,loss_correction2,loss_correction3}, we propose to incorporate loss correction strategy into graph self-training.  In brief, {\model} corrects the predictions of the student model in each iteration, so as to eliminate the negative impact of misleading pseudo labels from the teacher model.

Specifically, given a student model $f_{\bar{\theta}}$ trained by pseudo labels, the loss correction assumes there is a model $f_{\theta^*}$ trained by ground-truth labels 
and a \emph{transition matrix} $\mathbf{T}$ such that $f_{\bar{\theta}}$ can be represented by $f_{\bar{\theta}}=\mathbf{T}f_{\theta^*}$, as shown in Fig. \ref{fig:loss_correction}, 
where each element in $\mathbf{T}\in\mathbb{R}^{c\times c}$ is a transition probability from the ground-truth label to the pseudo label, \ie 
$T_{kj}=P(\bar{Y}=j|Y=k)$ and $c$ is the number of classes. With such a transition matrix, every model trained by pseudo labels is equal to that trained by ground-truth labels.
We have proved the equivalence relation above using the following proposition.



\begin{proposition}
\label{theory_lc}
  Given a model $f_{\bar{\theta}}$ trained by pseudo labels and a model $f_{{\theta}^*}$ trained by ground-truth labels, assuming that there exists a transition matrix $\mathbf{T}$ such that the equation $f_{\bar{\theta}}(\mathbf{x}_u,\mathbf{A})=\mathbf{T}f_{\theta^ *}(\mathbf{x}_u,\mathbf{A})$ holds for each node $v_u$, then $\bar{\theta}=\theta^*$ if $\mathbf{T}$ is a permutation matrix under cross entropy (CE) loss or $\mathbf{T}$ is an arbitrary non-zero matrix under mean square error (MSE) loss.
\end{proposition}
\begin{proof}
  Please refer to Appendix \ref{appendix:proof_2}.
\end{proof}

Based on Proposition \ref{theory_lc}, ideally, we can train the student model regardless of the quality of labels, 
and recover $f_{\theta^*}$ with $\mathbf{T}$. Specifically, as shown in Fig. \ref{fig:loss_correction}, for each node $v_i\in\{\mathcal{V}_L\cup\mathcal{S}_U\}$ with its 
feature vector $\mathbf{x}_i$, we first feed it into student model and multiply the output with $\mathbf{T}$ to get $f_{\bar{\theta}}(\mathbf{x}_i,\mathbf{A})$. 
Then we use $f_{\bar{\theta}}(\mathbf{x}_i,\mathbf{A})$ to optimize the student model according to Eq. \ref{eq:final_loss}. Finally, at inference, we can treat the student model as $f_{\theta^*}$. 
Please note that the transition matrix $\mathbf{T}$ is pre-computed 
and not updated during optimization of the student model.

Next, we make an illustration for the computation of the transition matrix $\mathbf{T}$. 
Noting that for each node $v_i\in\mathcal{V}_L$ with the ground-truth label $y_i=k$, the probability $P(Y=k|X=\mathbf{x}_i)$ should be 1 since we definitely know its label to be $k$. 
Therefore, given the output probability $p_{kj}=f_{\bar{\theta}}(\mathbf{x}_i,\mathbf{A})_j$ of class $j$, we have 
\begin{equation}
  \begin{aligned}
  p_{kj}&=P(\bar{Y}=j|X=\mathbf{x}_i)=\sum_{m=1}^c P(\bar{Y}=j|Y=m,X=\mathbf{x}_i)P(Y=m|X=\mathbf{x}_i)\\
  &=P(\bar{Y}=j|Y=k,X=\mathbf{x}_i)\cdot 1 + 0 + \cdots + 0 = T_{kj}(\mathbf{x}_i)=T_{kj}.
  \end{aligned}
\end{equation} 
In others words, the output probability vector $f_{\bar{\theta}}(\mathbf{x}_i,\mathbf{A})$ of each node $v_i$ with its ground-truth label $k$ is the $k$-th row of $\mathbf{T}$, where $\bar{\theta}$ means such a model is trained with the augmented dataset $\mathcal{V}_L\bigcup\mathcal{S}_U$.

Technically, we first train a student model $f_{\bar{\theta}}$ without loss correction using the augmented dataset $\mathcal{V}_L\cup \mathcal{S}_U$, then update $\mathbf{T}$ 
according to $p_{kj}=f_{\bar{\theta}}(\mathbf{x}_i,\mathbf{A})_j$, and finally re-train a student model from scratch with loss correction to obtain $f_{\theta^*}$.

Considering that there are multiple nodes belonging to class $k$ in $\mathcal{V}_L$, we propose the following optimization problem to learn $\mathbf{T}$ instead:  
\begin{equation}
  \label{eq:update_T}
  \arg\min_{\mathbf{T}}\sum_{k=1}^c\sum_{j=1}^{N_k^{(L)}}||\mathbf{T}_{k,:}-f_{\bar{\theta}}(\mathbf{x}_i,\mathbf{A})||^2+||\mathbf{T}\mathbf{T}^{\mathsf{T}}-\mathbf{I}||^2,
\end{equation}
where $N_k^{(L)}$ is the number of nodes belonging to class $k$ in $\mathcal{V}_L$ and $\mathbf{I}$ is an identity matrix. 
Since the improved CE loss is utilized as the loss function in this paper as mentioned in Eq. \ref{eq:objective_func} and Eq. \ref{eq:final_loss}, 
we append the regularization term $||\mathbf{T}\mathbf{T}^{\mathsf{T}}-\mathbf{I}||^2$ for guiding $\mathbf{T}$ to approximate to a permutation matrix, 
which is derived from Proposition \ref{theory_lc} under the CE loss. 
Moreover, we initialize $\mathbf{T}$ with the identify matrix $\mathbf{I}$ at the very beginning.

\begin{figure}
	\centering
	\includegraphics[width=1.0\columnwidth]{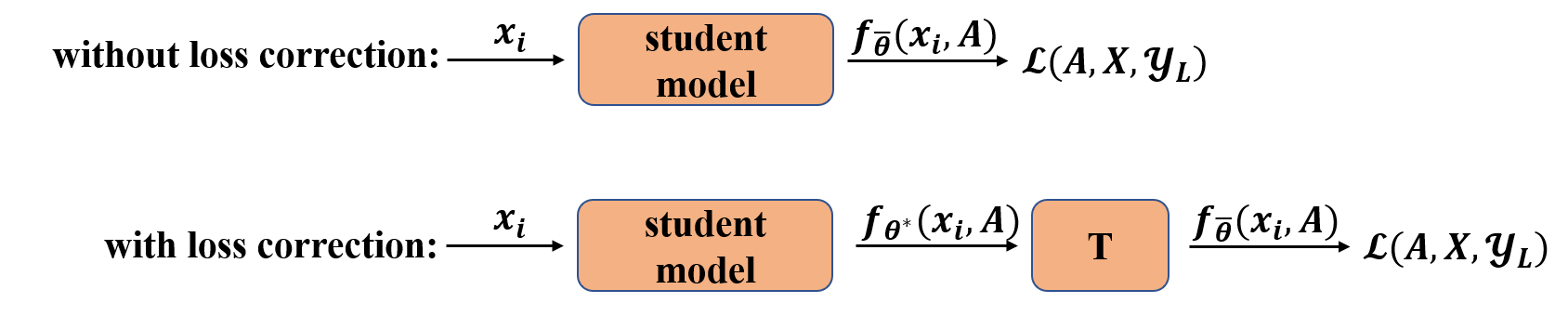}
	\caption{An illustration of loss correction.}
	\label{fig:loss_correction}
\end{figure}

\subsection{Overview of {\model}}
\label{subsec:overview}
Till now, we have elaborated our proposed {\model} framework, which solves both the {\ds} and the low-quality pseudo labels
with the help of information gain and loss correction. 
We summarize it in Algorithm \ref{algorithm} and further analyze its time complexity in Appendix \ref{appendix:time}.

Given a graph $\mathcal{G}=(\mathcal{V,E},\mathbf{X})$ with its original labeled dataset $\mathcal{V}_L$, unlabeled dataset $\mathcal{V}_U$, 
adjacent matrix $\mathbf{A}$ as well as its label set $\mathcal{Y}_L$, we first train a teacher model $f_\theta$ on $\mathcal{V}_L$ to obtain 
the prediction $\bar{y}_u$ and the confidence $r_u$ for each unlabeled node $v_u\in\mathcal{V}_U$ at line \ref{algo:teacher}.
then iterate steps from line \ref{algo:select} to \ref{algo:replace} util convergence, where we call each iteration a \emph{stage} following \cite{m3s}. 
Specifically, at line \ref{algo:select} we select part of unlabeled nodes whose confidence $r_u$ is bigger than a given threshold $\tau$ to obtain $\mathcal{S}_U$. 
Next at line \ref{algo:pseudo} we pseudo-label each node $v_u\in\mathcal{S}_U$ with $\bar{y}_u$ to augment $\mathcal{V}_L$. 
Then at line  \ref{algo:calculate} we calculate the information gain $\mathbb{B}_u$ according to dropout or dropedge variational inference in Section \ref{subsec:information_gain} 
and normalize it according to Eq. \ref{eq:final_loss}. 
With such information gain, we train a student model $f_{\bar{\theta}}$ at line \ref{algo:train} using the augmented dataset, where pseudo labels may be incorrect. 
Therefore, at line \ref{algo:update} we update the transition matrix $\mathbf{T}$ with the output probability vector of $f_{\bar{\theta}}$ of each node $v_i\in\mathcal{V}_L$ according to Eq. \ref{eq:update_T}, 
and retrain the student model from scratch at line \ref{algo:retrain} with $f_{\bar{\theta}}=\mathbf{T}f_{{\theta}^*}$ to get $f_{{\theta}^*}$. 
Finally, we replace the teacher model $f_\theta$ with $f_{{\theta}^*}$ and repeat above steps utill convergence.

\begin{algorithm}
  \caption{The {\model} Framework}
  \label{algorithm}
  \begin{algorithmic}[1]
    \REQUIRE Graph $\mathcal{G}=(\mathcal{V,E},\mathbf{X})$, original labeled dataset $\mathcal{V}_L$, unlabeled dataset $\mathcal{V}_U$, adjacent matrix $\mathbf{A}$, label set $\mathcal{Y}_L$, transition matrix $\mathbf{T=I}$
    \ENSURE Probability vector $\mathbf{p}_i$ for each node $v_i$
    \STATE Train a teacher model $f_\theta$ on $\mathcal{V}_L$ to obtain the prediction $\bar{y}_u$ and the confidence $r_u$ for each unlabeled node $v_u\in\mathcal{V}_U$;
    \label{algo:teacher}
    \FOR{each stage $k$}
    \STATE Select part of unlabeled nodes according to $r_u$ to get $\mathcal{S}_U$;
    \label{algo:select}
    \STATE Pseudo-labeling each node $v_u\in\mathcal{S}_U$ with $\bar{y}_u$;
    \label{algo:pseudo}
    \STATE Calculate the information gain $\mathbb{B}_u$ according to Eq. \ref{eq:information_gain};
    \label{algo:calculate}
    \STATE Train a student model $f_{\bar{\theta}}$ without $\mathbf{T}$ according to Eq. \ref{eq:final_loss};
    \label{algo:train}
    \STATE Update $\mathbf{T}$ using $f_{\bar{\theta}}(\mathbf{x}_i,\mathbf{A})$ of $v_i\in\mathcal{V}_L$ according to Eq. \ref{eq:update_T};
    \label{algo:update}
    \STATE Retrain a student model from scratch according to Eq. \ref{eq:final_loss} with  $f_{\bar{\theta}}=\mathbf{T}f_{{\theta}^*}$ to get $f_{{\theta}^*}$;
    \label{algo:retrain}
    \STATE Replace the teacher model $f_\theta$ with the student model $f_{{\theta}^*}$;
    \label{algo:replace}
    \ENDFOR
    \RETURN $\mathbf{p}_i=f_{{\theta}^*}(\mathbf{x}_i,\mathbf{A})$ in the final stage.
  \end{algorithmic}
\end{algorithm}

\subsection{Theoretical Analysis}
\label{subsec:theory}
In this section, we theoretically analyze the influence factors on self-training from the perspective of gradient descent, and our theorem below demonstrates the rationality of the whole DR-GST framework.

\begin{theorem}
  \label{theory}
    Assuming that $||\nabla_\theta l(y_i,\mathbf{p}_i)||\leqslant \Psi$ for each node $v_i$, where $\Psi$ is a constant, given $\nabla_\theta\mathcal{L}_{pop}$ and $\nabla_\theta\mathcal{L}_{st}$, the gradient of $\mathcal{L}_{pop}$ and $\mathcal{L}_{st}$ w.r.t. model parameters $\theta$,
    the following bound between $\nabla_\theta \mathcal{L}_{pop}$ and $\nabla_\theta \mathcal{L}_{st}$ holds: 
    \begin{equation}
        \begin{aligned}
            ||\nabla_\theta \mathcal{L}_{pop} - \nabla_\theta \mathcal{L}_{st}|| &\leqslant \frac{|\mathcal{S}_U|}{|\mathcal{V}_L\cup\mathcal{S}_U|}\Psi(2||P_{(v_u,y_u)\thicksim P_{pop}(\mathcal{V,Y})}(\bar y_u\neq y_u)||\\
            &+||P_{st}(\mathcal{V,Y})-P_{pop}(\mathcal{V,Y})||).
        \end{aligned}
    \end{equation}
\end{theorem}
\begin{proof}
    Please refer to Appendix \ref{appendix:proof_1}.
\end{proof}

From the Theorem \ref{theory} we can conclude that the performance of self-training is negatively related to the difference $||P_{st}(\mathcal{V,Y})-P_{pop}(\mathcal{V,Y})||$ between the two distributions 
as well as the error rate $||P_{(v_u,y_u)\thicksim P_{pop}(\mathcal{V,Y})} (\bar y_u\neq y_u)||$ of pseudo labels. Meanwhile, we find our proposed {\model} is a natural framework equipped with two designs to correspondingly address the issues in self-training: \emph{information gain weighted loss function} for distribution recovery and \emph{loss correction strategy} for improving qualities of pseudo labels. This analysis further demonstrates the rationality of {\model} framework from the theoretical perspective.

\section{Experiment}
In this section, we evaluate the effectiveness of {\model} framework on semi-supervised node classification task with five widely used benchmark datasets from citation networks~\cite{cora,corafull} (\ie Cora, Citeseer, Pubmed and CoraFull) and social networks~\cite{flickr} (\ie Flickr). More detailed descriptions about datasets are in Appendix \ref{appendix:dataset}.

\subsection{Experimental Setup}



\begin{table*}[]
  \caption{Node classification results(\%). (L/C: the number of labels per class; bold: best)}
  \label{tab:main_result}
  \setlength{\tabcolsep}{0.5mm}{
      \begin{tabular}{@{}c|cccc|cccc|cccc|cccc|cccc@{}}
      \toprule
      Dataset & \multicolumn{4}{c|}{Cora}     & \multicolumn{4}{c|}{Citeseer}  & \multicolumn{4}{c|}{Pubmed}    & \multicolumn{4}{c|}{CoraFull}  & \multicolumn{4}{c}{Flickr}    \\ \midrule
      L/C     & 3     & 5     & 10    & 20    & 3     & 5     & 10    & 20    & 3     & 5     & 10    & 20    & 3     & 5     & 10    & 20    & 3     & 5     & 10    & 20    \\ \hline
      GCN     & 64.52 & 69.55 & 78.03 & 81.56 & 51.39 & 61.34 & 68.39 & 71.64 & 66.04 & 71.25 & 75.88 & 79.31 & 41.83 & 49.12 & 55.67 & 60.69 & 37.69 & 40.64 & 48.04 & 51.74 \\
      GAT     & 67.19 & 69.45 & 76.38 & 82.24 & 55.19 & 59.40 & 67.61 & 72.00 & 67.85 & 68.41 & 72.42 & 78.38 & 36.44 & 46.70 & 52.45 & 57.97 & 20.02 & 24.90 & 33.27 & 37.06      \\
      APPNP   & 65.06 & 75.53 & 81.33 & 83.14 & 51.22 & 60.48 & 68.50 & 71.64 & 65.77 & 73.01 & 76.35 & 79.51 & 40.29 & 44.49 & 50.89 & 60.77 & 24.76 & 35.54 & 47.87 & \textbf{61.55}      \\ \hline
      STs     & 70.68 & 75.60 & 80.35 & 82.89 & 56.29 & 65.59 & 74.17 & 74.36 & 69.82 & 73.77 & 77.68 & 81.02 & 43.44 & 51.16 & 58.40 & 61.70 & 35.21 & 43.25 & 48.23 & 52.99 \\
      M3S     & 64.24 & 71.02 & 78.93 & 82.78 & 50.07 & 63.28 & 74.54 & 74.72 & 68.76 & 69.21 & 70.72 & \textbf{81.34} & 42.77 & 49.75 & 57.43 & 61.40 & 35.33 & 39.02 & 47.62 & 51.87 \\
      ABN     & 66.39 & 73.07 & 78.73 & 81.79 & 54.30 & 64.27 & 69.90 & 72.81 & 59.17 & 71.40 & 75.26 & 79.09 & 43.38 & 48.39 & 55.88 & 60.62 & 35.13 & 41.62 & 47.01 & 52.10 \\ \hline
      {\model}$_{do}$      & 70.85 & \textbf{77.92} & 80.88 & 83.34 & 59.39 & 69.08 & \textbf{75.00} & \textbf{75.78} & \textbf{70.74} & \textbf{74.63} & \textbf{78.44} & 81.08 & \textbf{45.44} & \textbf{53.29} & \textbf{60.01} & 62.75 & 37.84 & \textbf{43.47} & \textbf{49.48} & 53.66 \\
      {\model}$_{de}$      & \textbf{73.43} & 77.59 & \textbf{81.67} & \textbf{84.03} & \textbf{60.60} & \textbf{69.91} & 74.65 & 75.26 & 70.55 & 73.71 & 77.42 & 80.65 & 45.42 & 52.50 & 59.16 & \textbf{63.11} & \textbf{38.21} & 43.28 & 49.44 & 53.05 \\ \bottomrule
      \end{tabular}}
\end{table*}

\subsubsection{Baselines}
We compare our proposed {\model} framework with two categories of baselines, including three representative GCNs (\ie GCN \cite{gcn}, GAT \cite{gat}, PPNP \cite{ppnp}) 
and three graph self-training frameworks (\ie STs \cite{deeper}, M3S \cite{m3s}, ABN \cite{abn}). Noting that STs includes four variants (\ie Self-Training, Co-Training, Union and Intersection) in the original paper and the best performance is reported in our experiments. The implementation of DR-GST and all the baselines can be seen in Appendix \ref{appendix:implementation}. More detailed experimental environment can be seen in Appendix \ref{appendix:environment}.

\subsubsection{Evaluation Protocol}
To more comprehensively evaluate our model, for all the datasets, 
we arrange only a few (including 3, 5, 10, 20) labeled nodes per class ($L/C$) for the training set following \cite{deeper}. 
Specifically, in the setting $L/C=20$, we follow the standard split \cite{cora} for Cora, Citeseer and Pubmed, 
and manually select 20 labeled nodes per class for CoraFull and Flickr considering the lack of standard split. 
In the setting $L/C<20$, we make 10 random splits for each $L/C$, where each random split represents that we randomly select part of nodes from the 
training set of $L/C=20$. 
For all the methods and all the cases, we run 10 times and report the mean accuracy.



\subsection{Overall Comparison on Node Classification}
The performance of different methods on node classification are summarized in Table \ref{tab:main_result}. We have the following observations.

\begin{itemize}[leftmargin=*]
\item Our proposed {\model} framework outperforms all the baselines by a considerable margin across most cases of all the datasets. The results demonstrate the effectiveness of {\model} by adopting a more principled mechanism to make use of unlabeled nodes in graph for boosting  classification performance.

\item With the decrease of labeled nodes, we observe that the performance of GCNs (\ie GCN, GAT and APPNP) drops quickly.  For clarity, we further illustrate the changing trend of accuracy \wrt $L/C$ in Fig. \ref{fig:acc_l_c}. Obviously, we can discover the larger performance margin between {\model} and GCNs with fewer labeled nodes per class, which further implies the superior capacity of {\model} for addressing labeled data scarcity on graph learning.

\item Considering the two variants of {\model}, we find that {\model}$_{do}$ performs better on Pubmed, CoraFull and Flickr while {\model}$_{de}$ on Cora and Citeseer. An intuitive explanation for such distinct performance is the different emphasis on network topology and feature information  \wrt different graphs for node classification task. Correspondingly, in {\model} framework, MC-dropedge performs information gain estimation with network topology while MC-dropout is based on feature information. This finding also sheds light on possible future work to combine both topology and feature to further enhance performance under our framework.

\item Among the two categories of baselines, self-training frameworks (\ie STs, M3S and ABN) can generally improve GCNs (\ie GCN, GAT and APPNP), which indicates the usefulness of unlabeled data. Nevertheless, {\model} still yields better performance for the following two promising designs: 
1) We pay more attention on nodes with high information gain rather than high confidence, so that the unsatisfying {\ds} issue is avoided.
2) We adopt a loss correction strategy, where qualities of pseudo labels are improved for subsequent self-training.

\end{itemize}

\subsection{In-depth Analysis of {\model}}
In this section, we make a series of analysis to better understand each component in {\model}, as well as key parameter selections.

\subsubsection{Ablation Study}
As mentioned above, the performance of self-training theoretically hinges on the distribution gap and qualities of pseudo labels, which could be naturally captured by our {\model} framework with two corresponding designs: the information-gain based weighted loss function and loss correction module. To comprehensively understand their contributions towards self-training on graphs, we prepare following three variants of {\model}:
\begin{itemize}[leftmargin=*]
  \item \textbf{{\model}-lc}: {\model} only with the loss correction module, \ie $\bar{\mathbb{B}}=1$ for all the unlabeled nodes.
  \item \textbf{{\model}-ig}: {\model} only with the information gain weighted loss function.
  \item \textbf{{\model}-w/o}: {\model} without the above two designs.
\end{itemize}

The results on {\model}$_{do}$ and {\model}$_{de}$ are respectively reported in Fig. \ref{fig:ablation_do} and Fig. \ref{fig:ablation_de}
From the results we can find that the overall performance order is as follows: {\model} $>$ {\model}-ig $>$ {\model}-lc  $>$ {\model}-w/o. There are three conclusions here. 
Firstly, the best performance achieved by the complete {\model} framework indicates the effectiveness of considering two components together. 
Secondly, the information gain weighted loss function and loss correction are both value modules for self-training on graphs. Thus, ignoring them altogether (\ie {\model}-w/o) is not ideal. 
Thirdly, the information-gain weighted loss function plays a more vital role in our self-training framework since {\model}-lc generally does not perform as well as {\model}-ig. 
In short, above findings further verify the rationality of {\model} from the empirical perspective. 

\subsubsection{Parameter Study}

\begin{figure}[t]
	\centering
	\subfigure{\includegraphics[width=0.32\columnwidth]{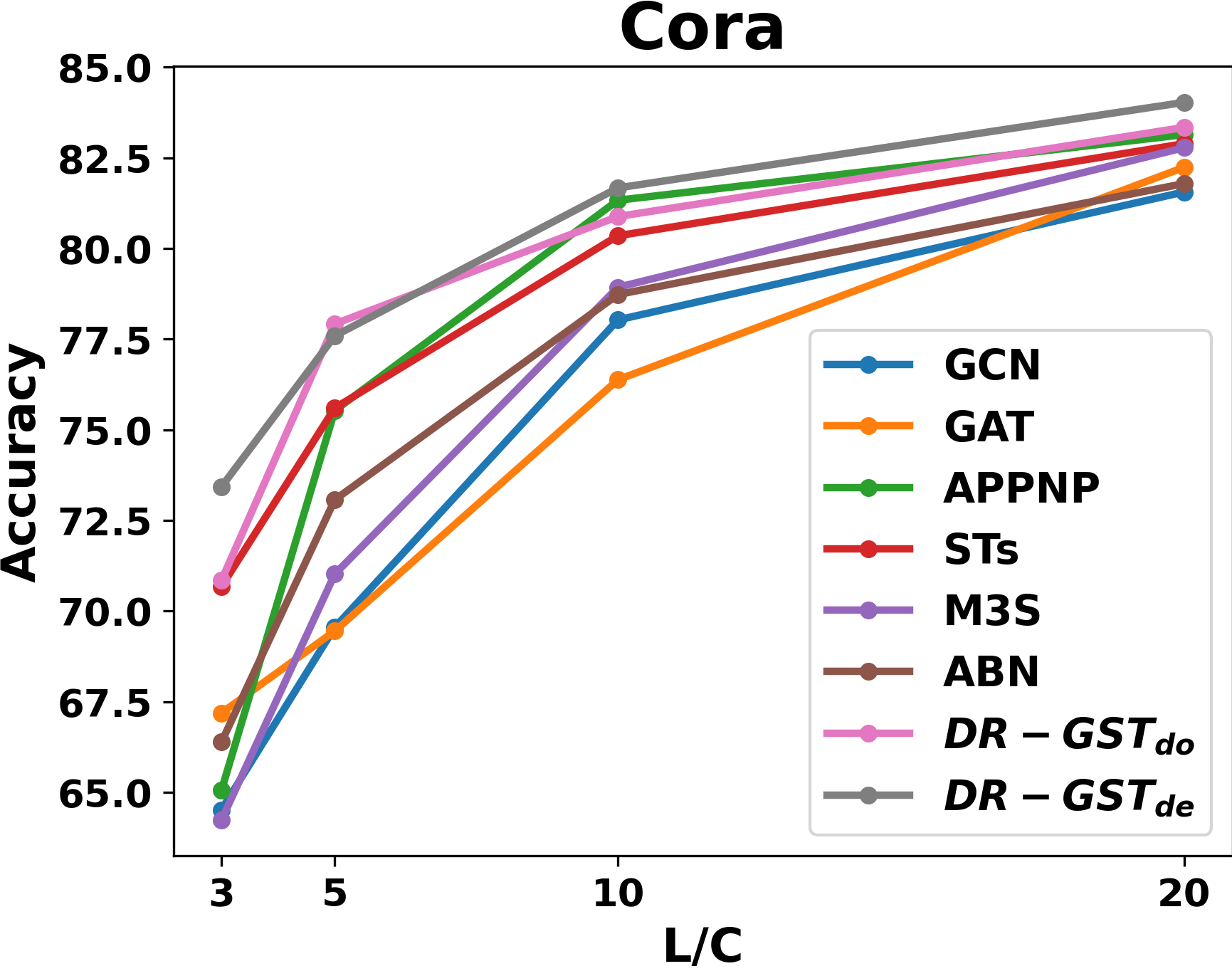}}
	\subfigure{\includegraphics[width=0.32\columnwidth]{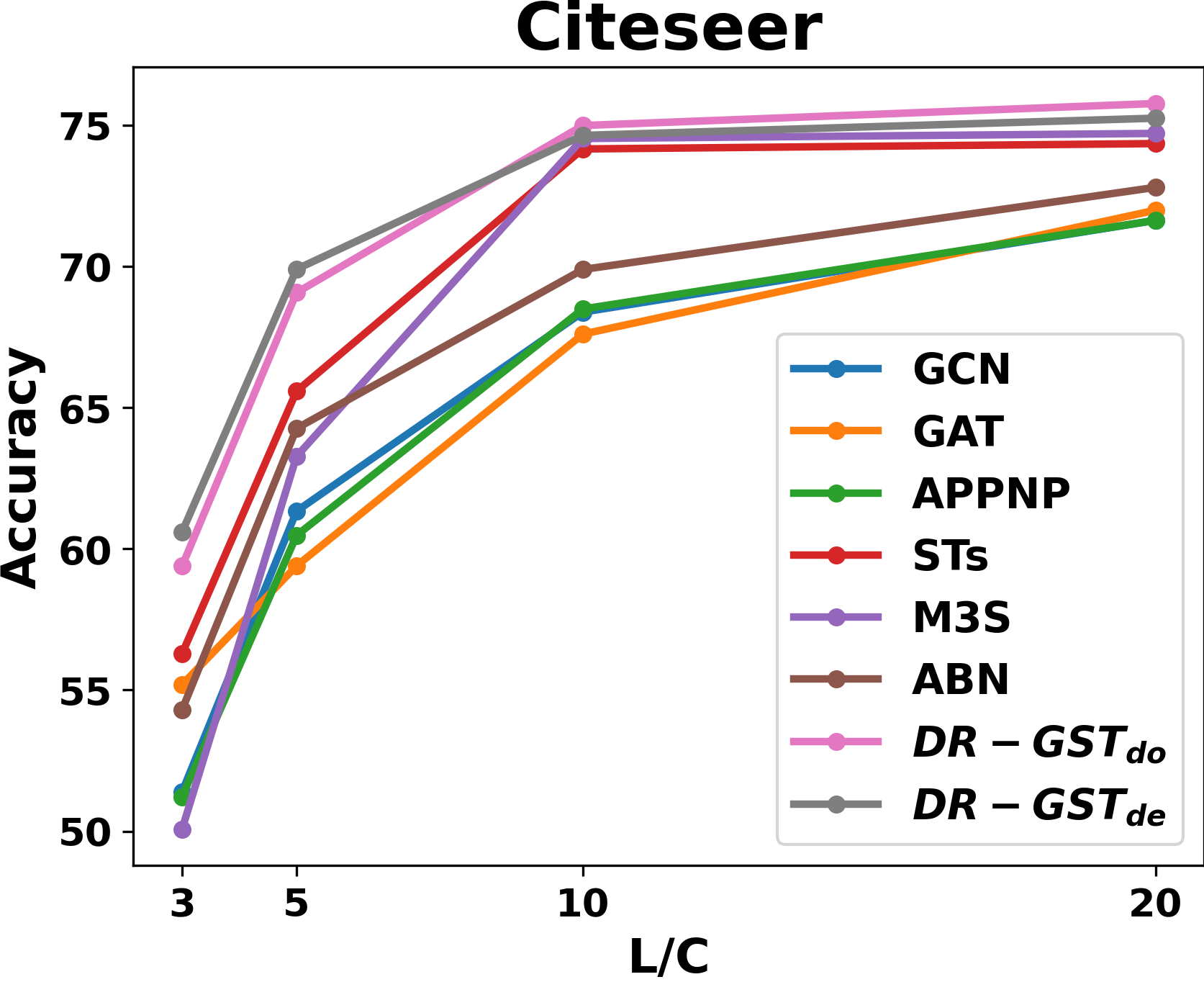}}
  \subfigure{\includegraphics[width=0.32\columnwidth]{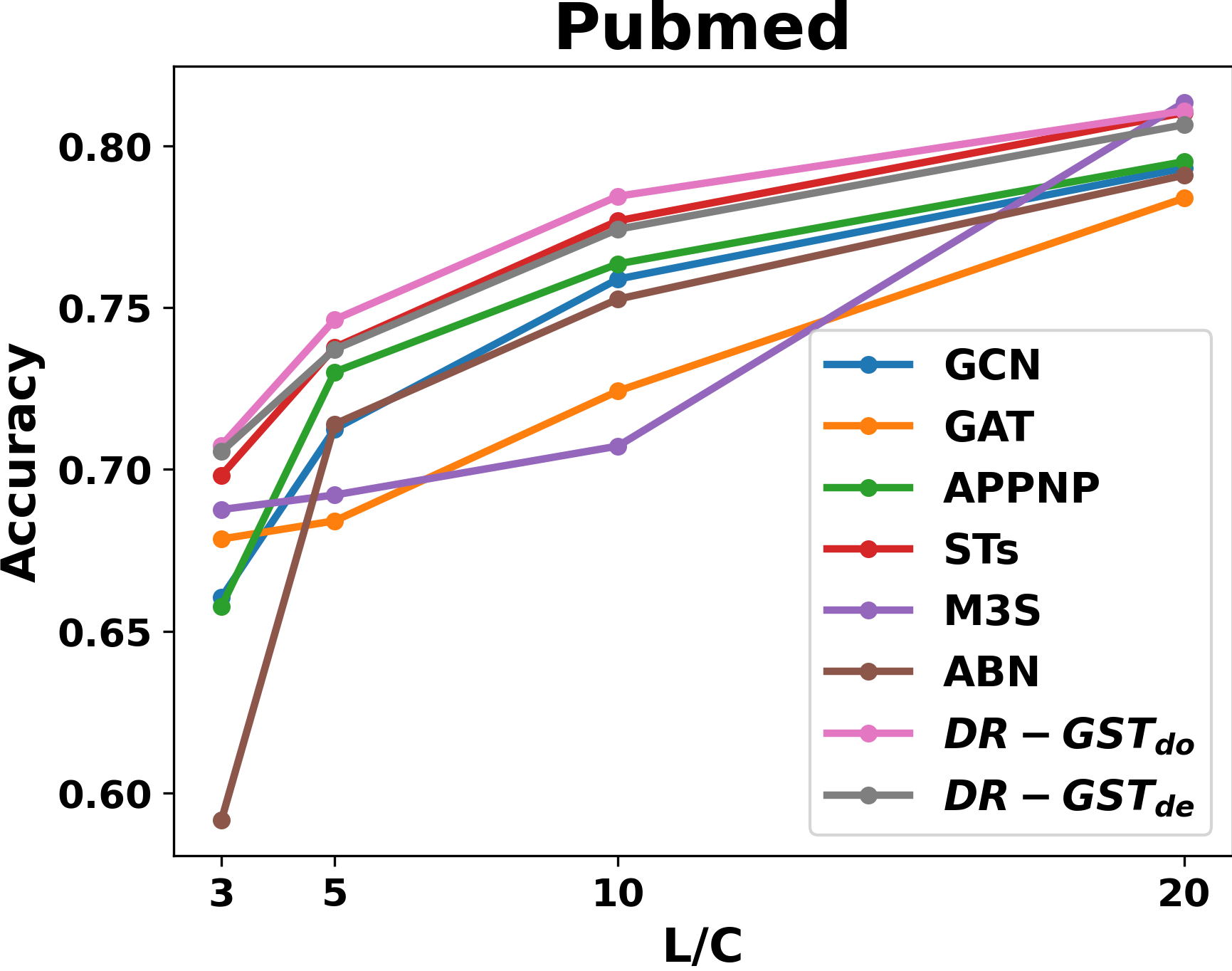}}
  
	\subfigure{\includegraphics[width=0.32\columnwidth]{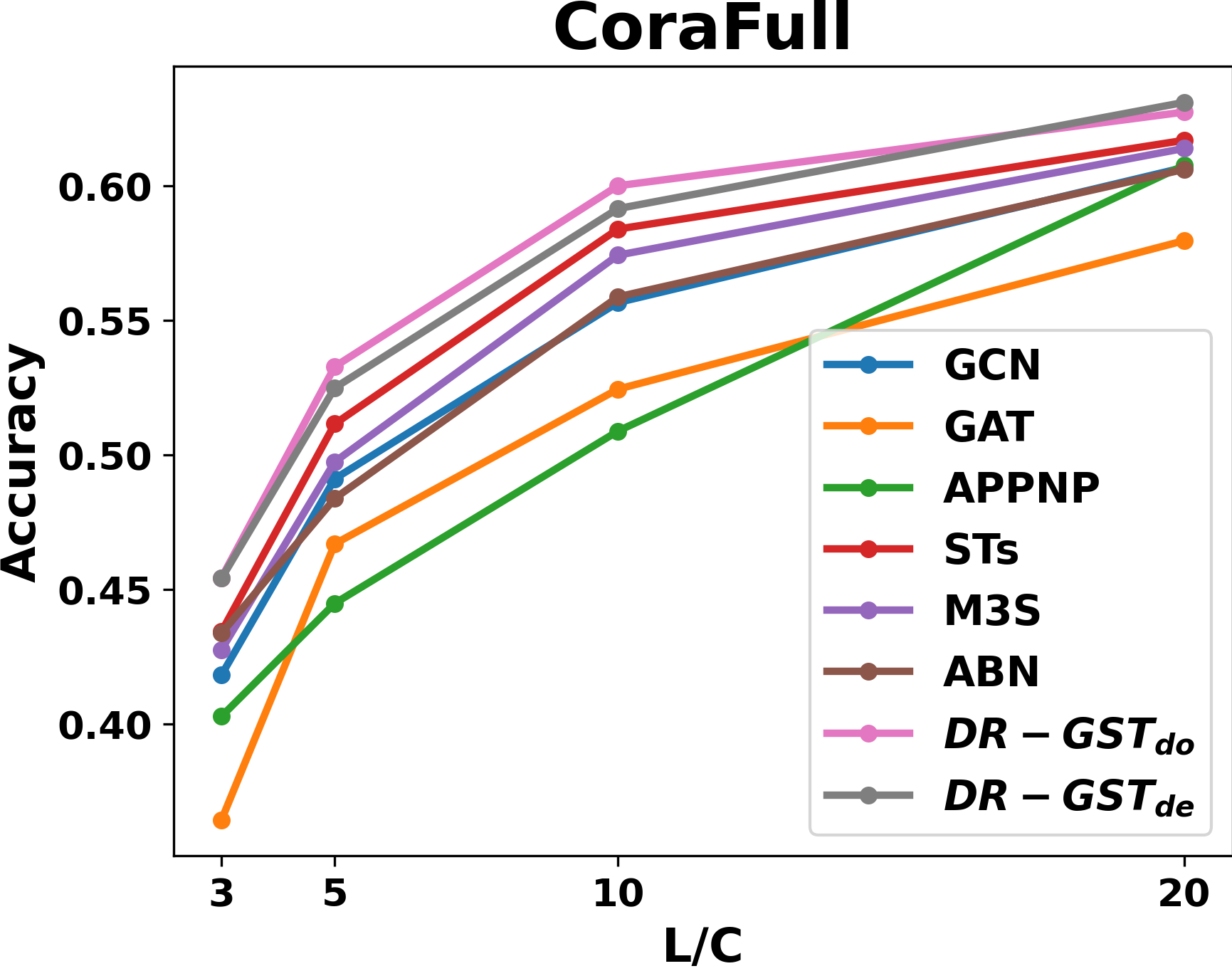}}
  \subfigure{\includegraphics[width=0.32\columnwidth]{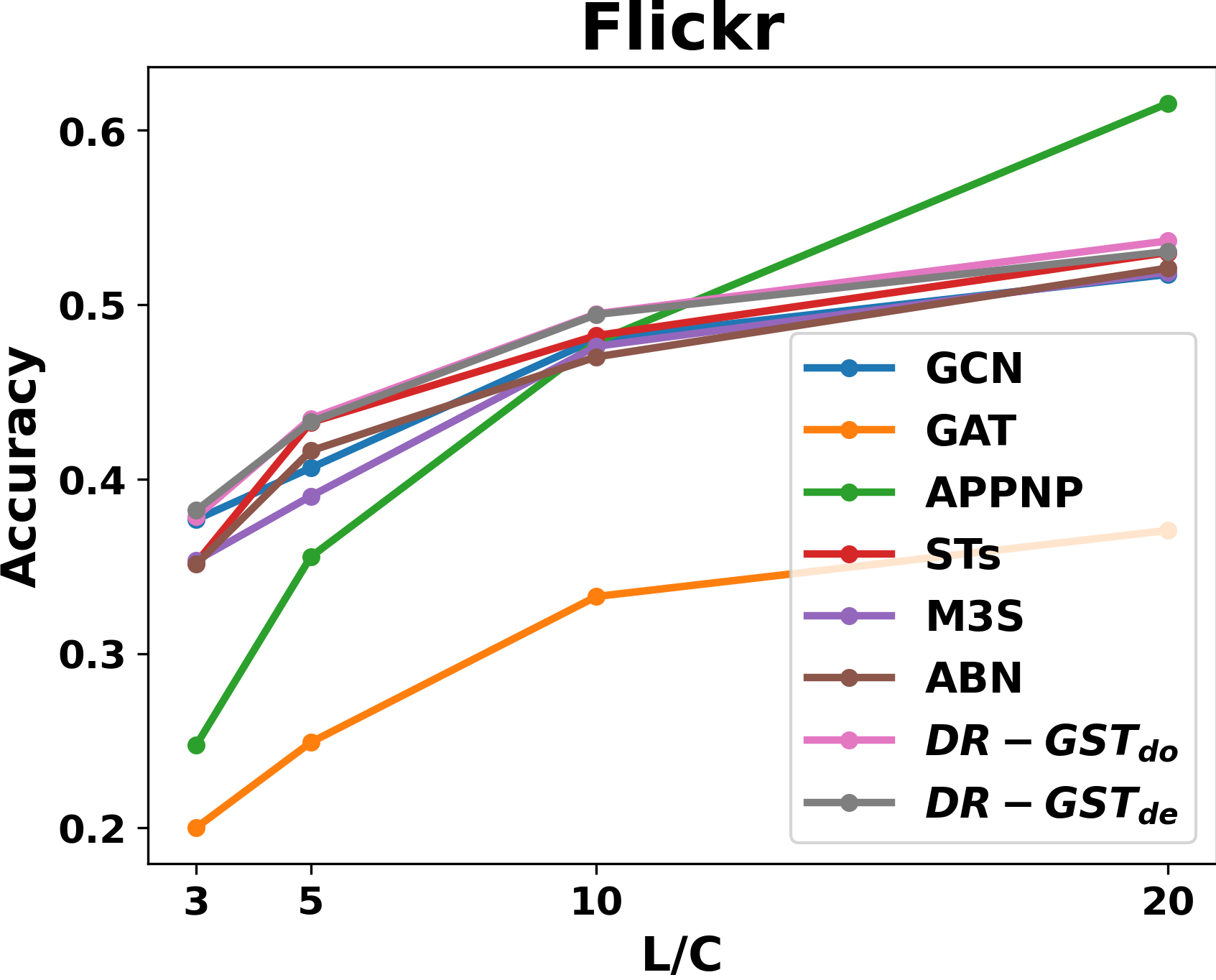}}
  \caption{The changing trends of accuracy w.r.t. $L/C$}
	\label{fig:acc_l_c}
\end{figure}

\begin{figure}
	\centering
	\subfigure{\includegraphics[width=0.43\columnwidth, height=2.5cm]{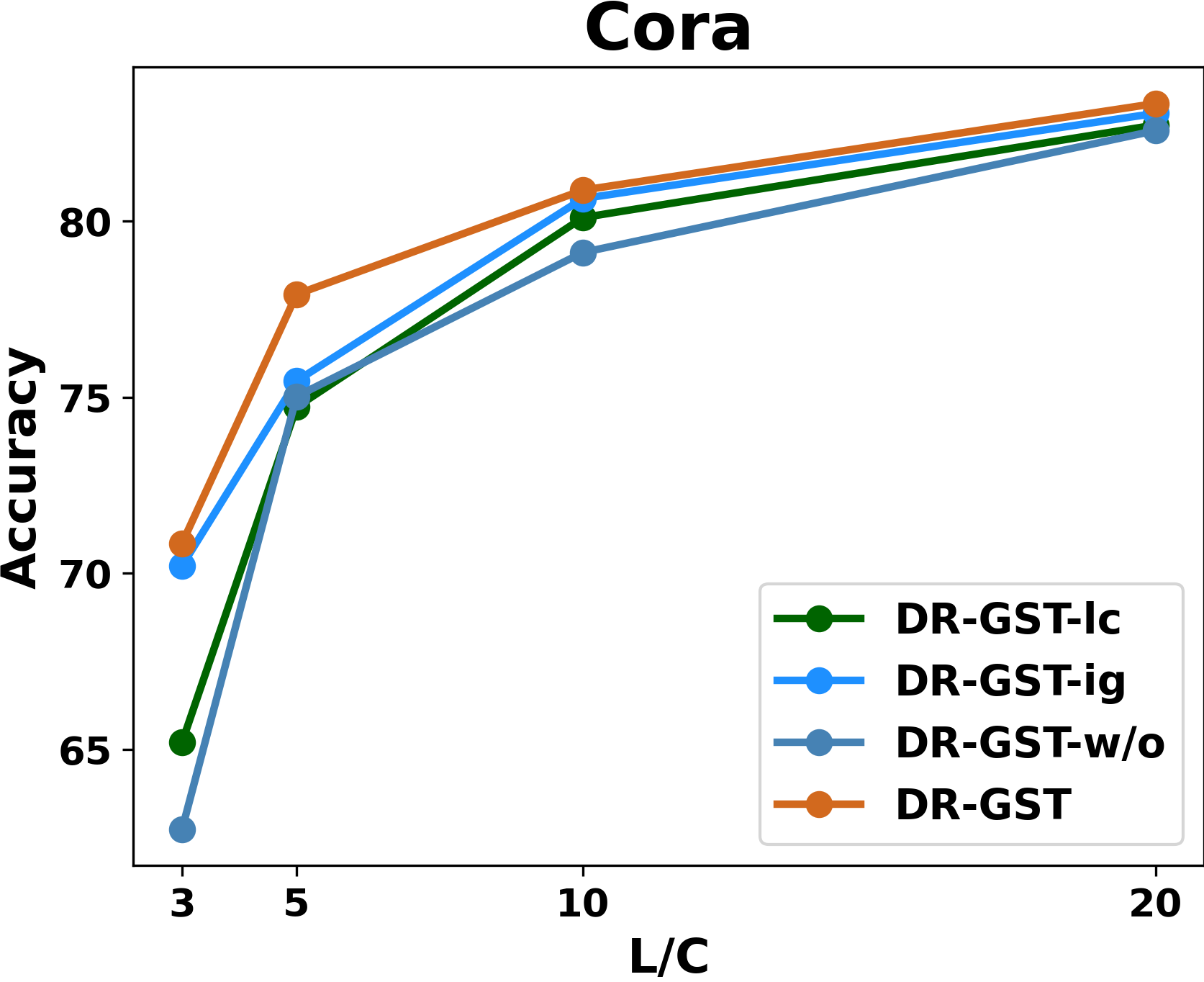}}
	\subfigure{\includegraphics[width=0.43\columnwidth,height=2.5cm]{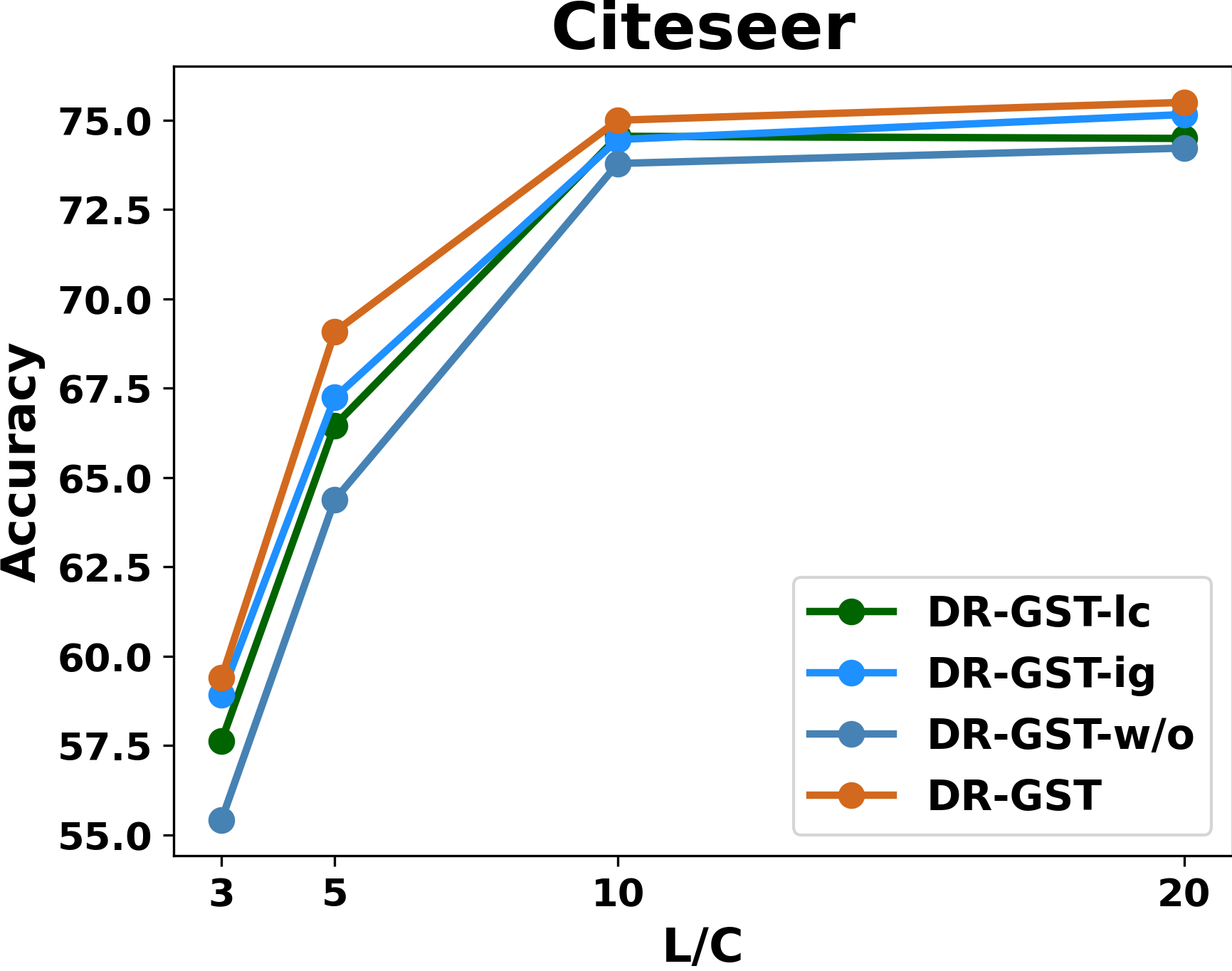}}

  \subfigure{\includegraphics[width=0.43\columnwidth,height=2.5cm]{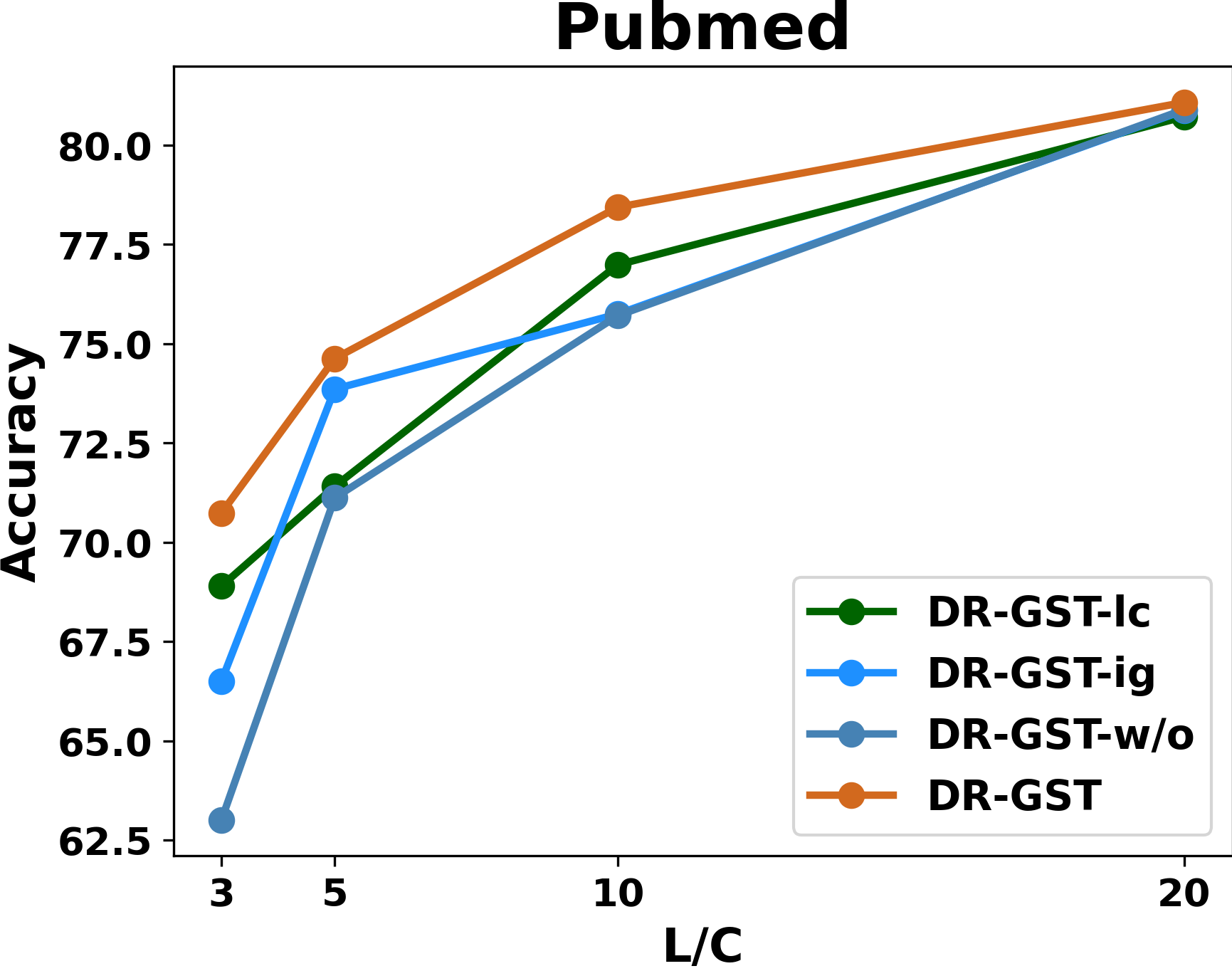}}
  \subfigure{\includegraphics[width=0.43\columnwidth,height=2.5cm]{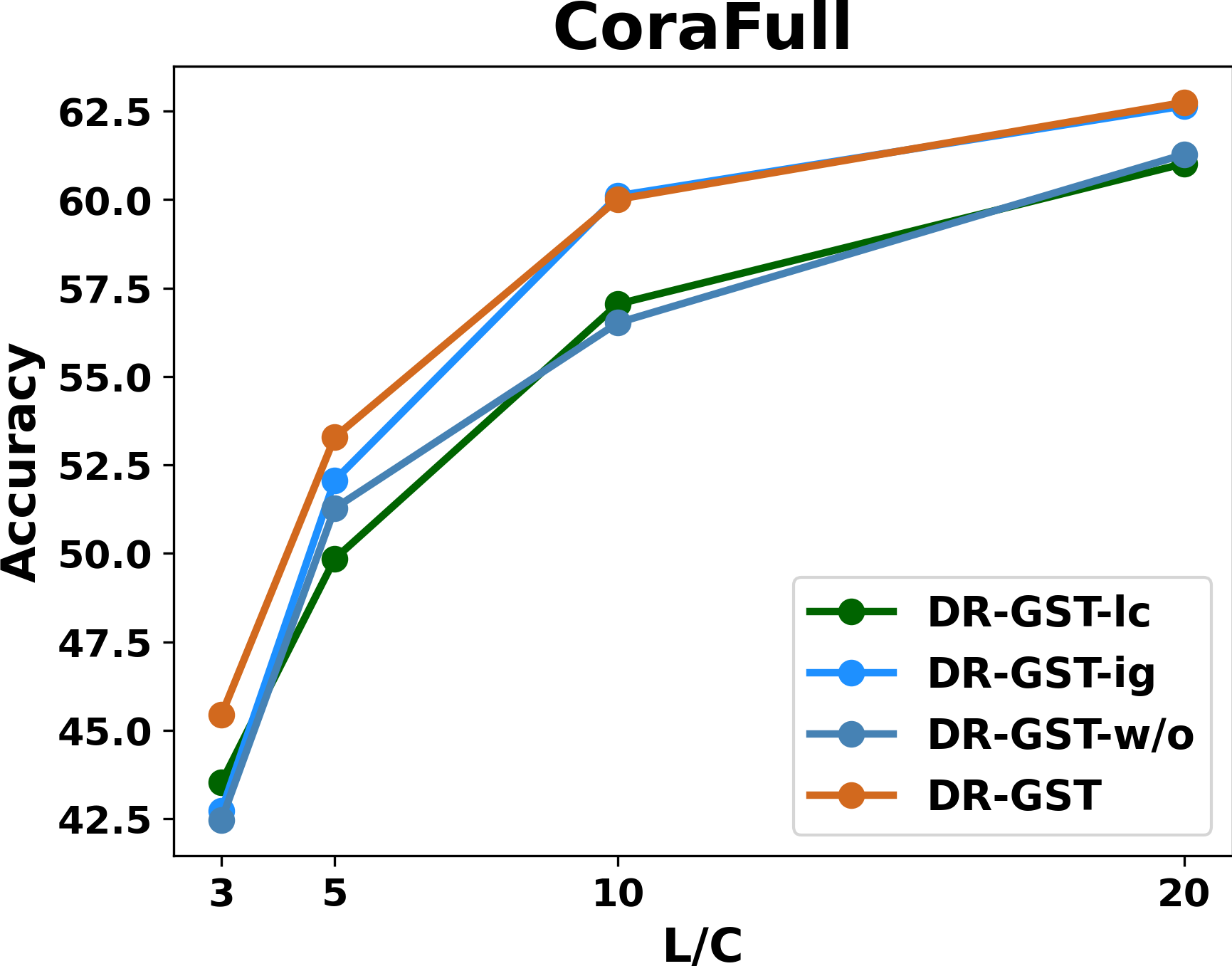}}
  \caption{Ablation study of {\model}$_{do}$.}
	\label{fig:ablation_do}
\end{figure}

\begin{figure}
	\centering
	\subfigure{\includegraphics[width=0.43\columnwidth,height=2.5cm]{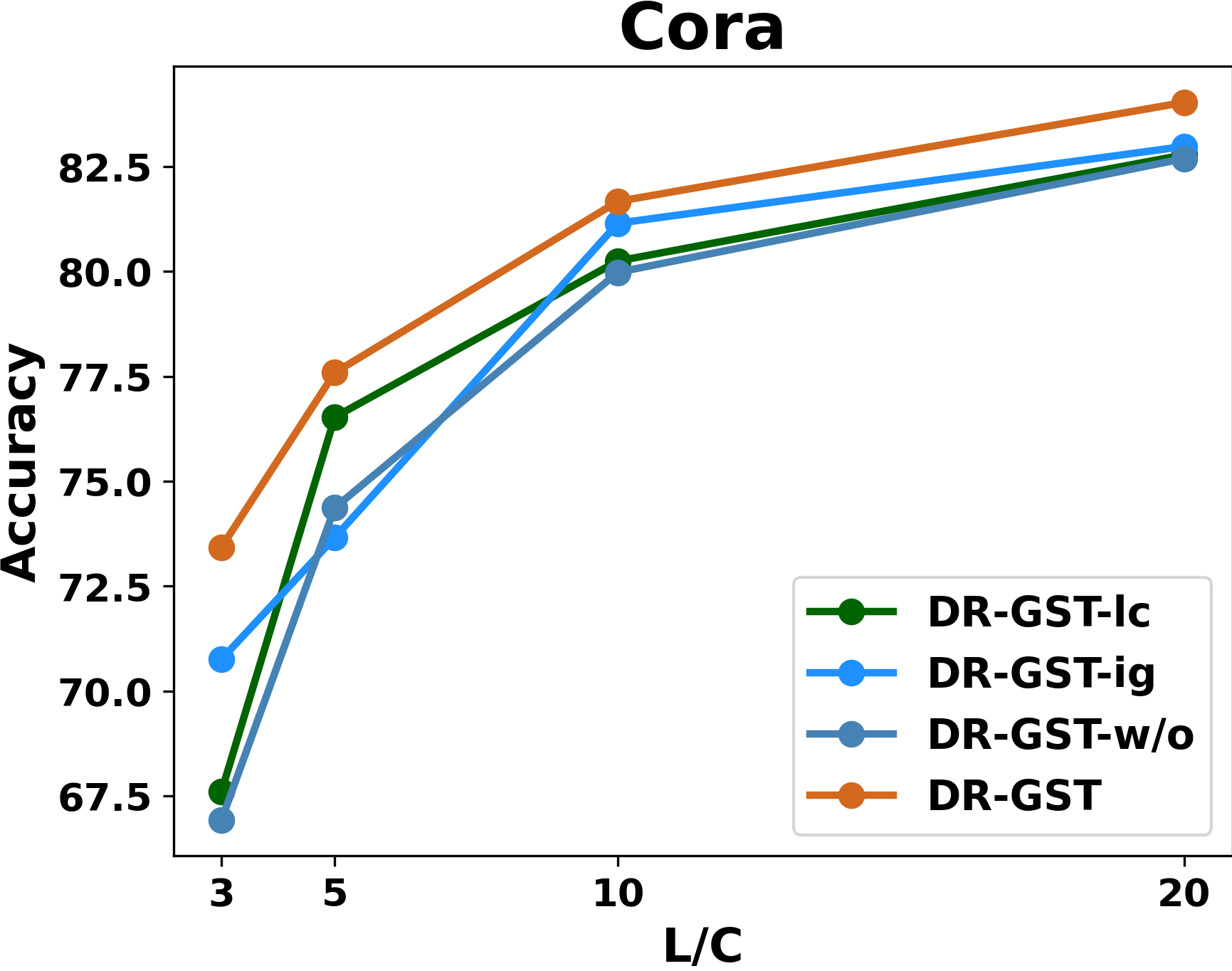}}
	\subfigure{\includegraphics[width=0.43\columnwidth,height=2.5cm]{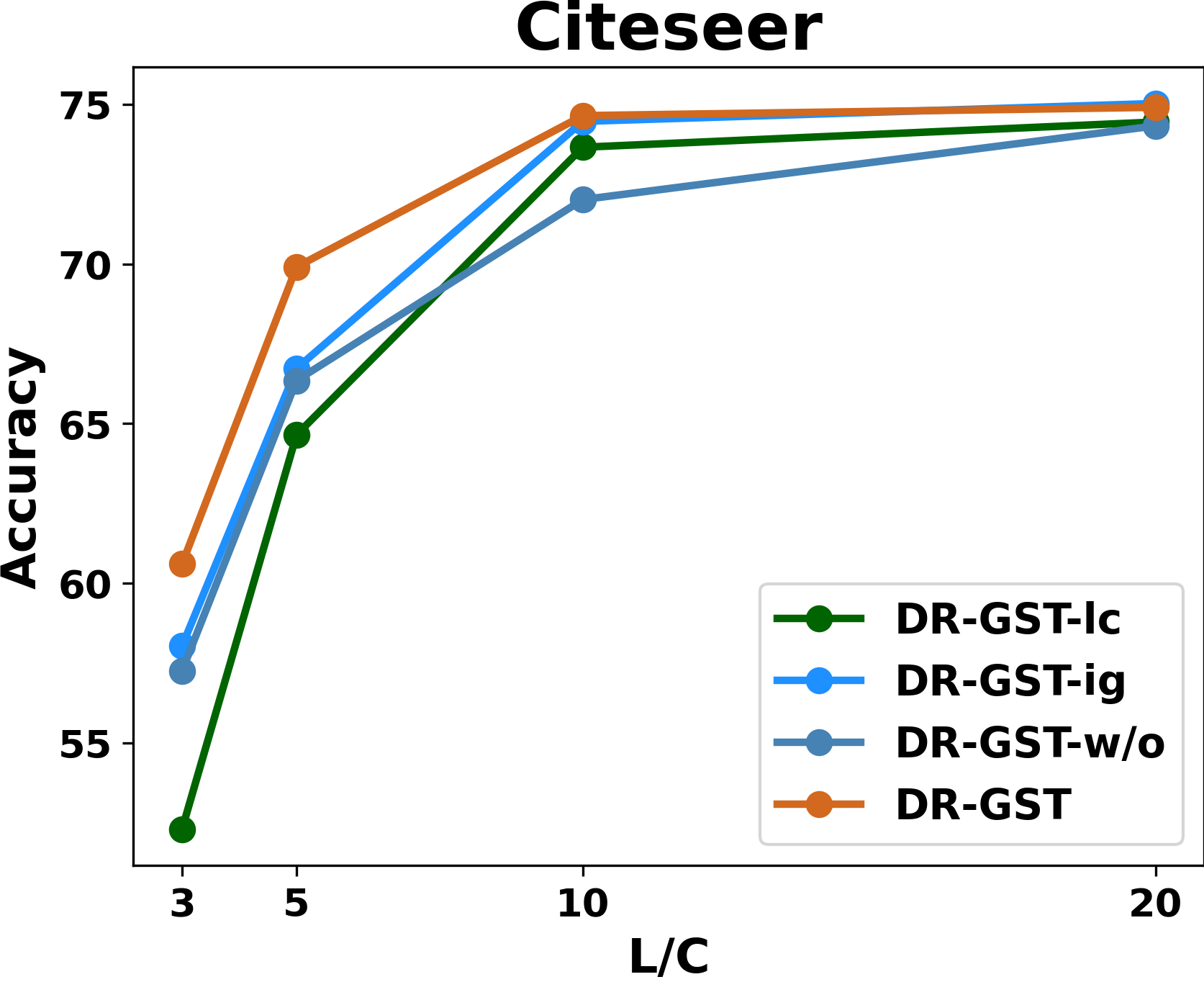}}

  \subfigure{\includegraphics[width=0.43\columnwidth,height=2.5cm]{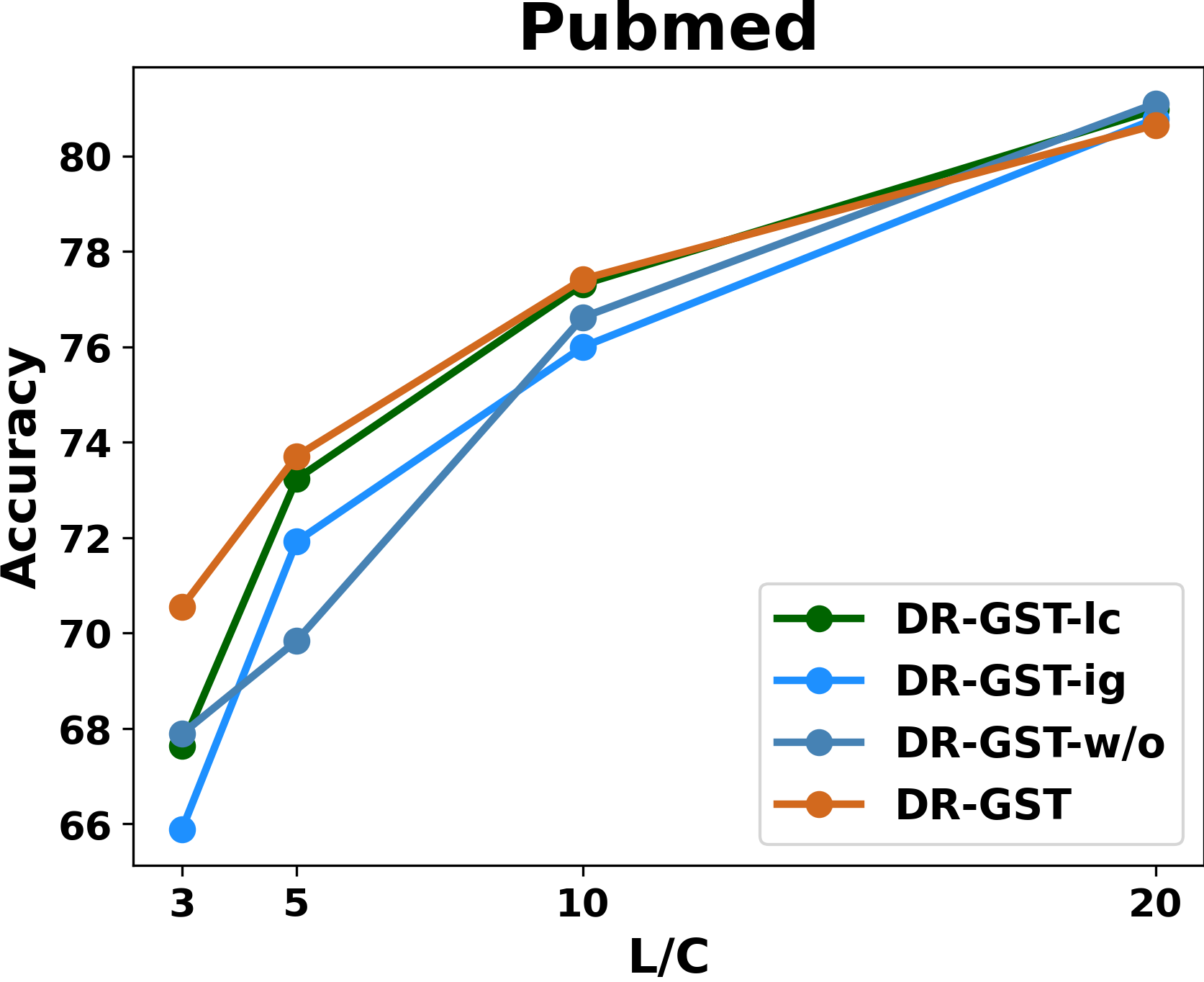}}
  \subfigure{\includegraphics[width=0.43\columnwidth,height=2.5cm]{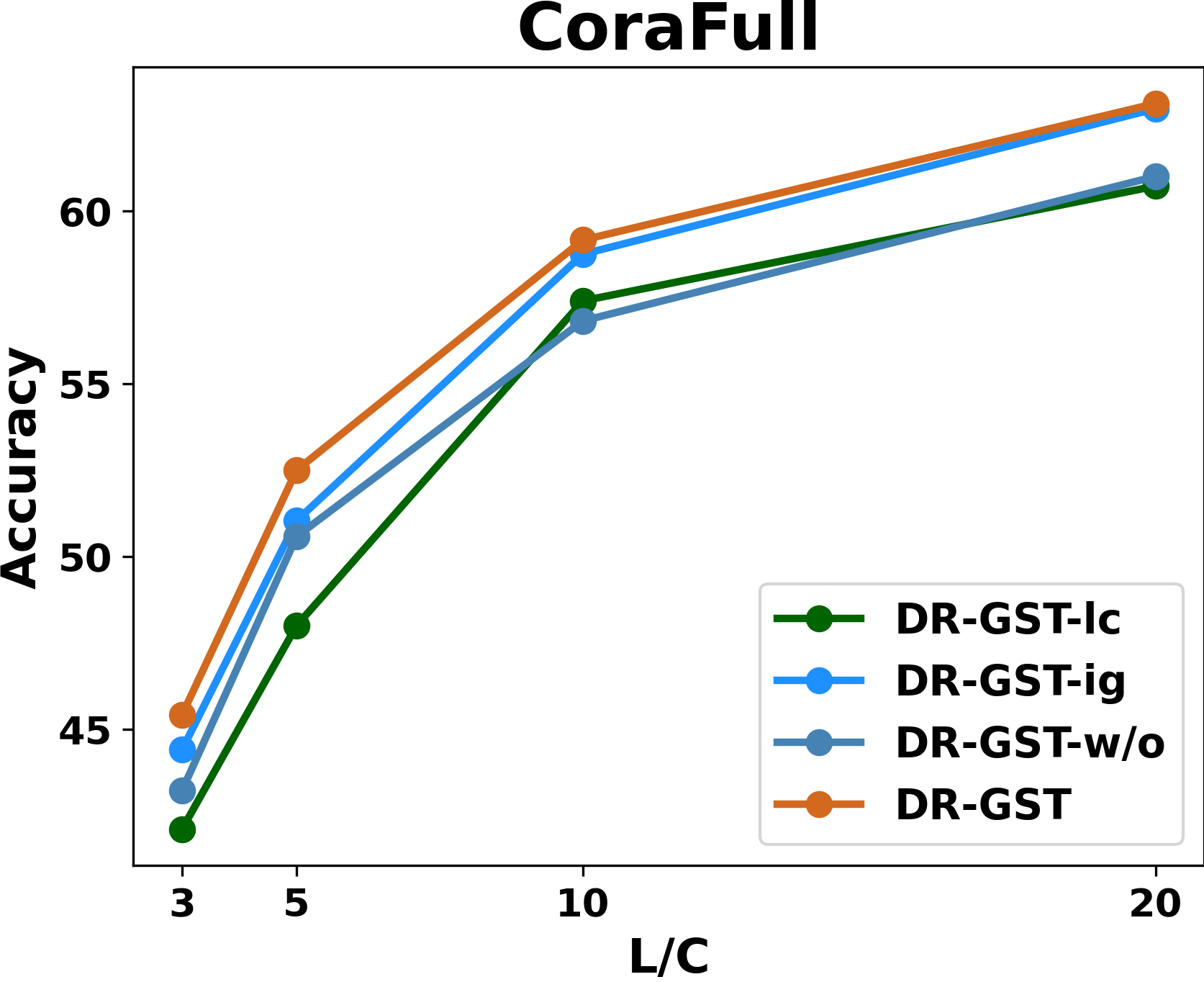}}
  \caption{Ablation study of {\model}$_{de}$.}
	\label{fig:ablation_de}
\end{figure}

\begin{figure}
	\centering
	\includegraphics[width=0.86\columnwidth, height=3cm]{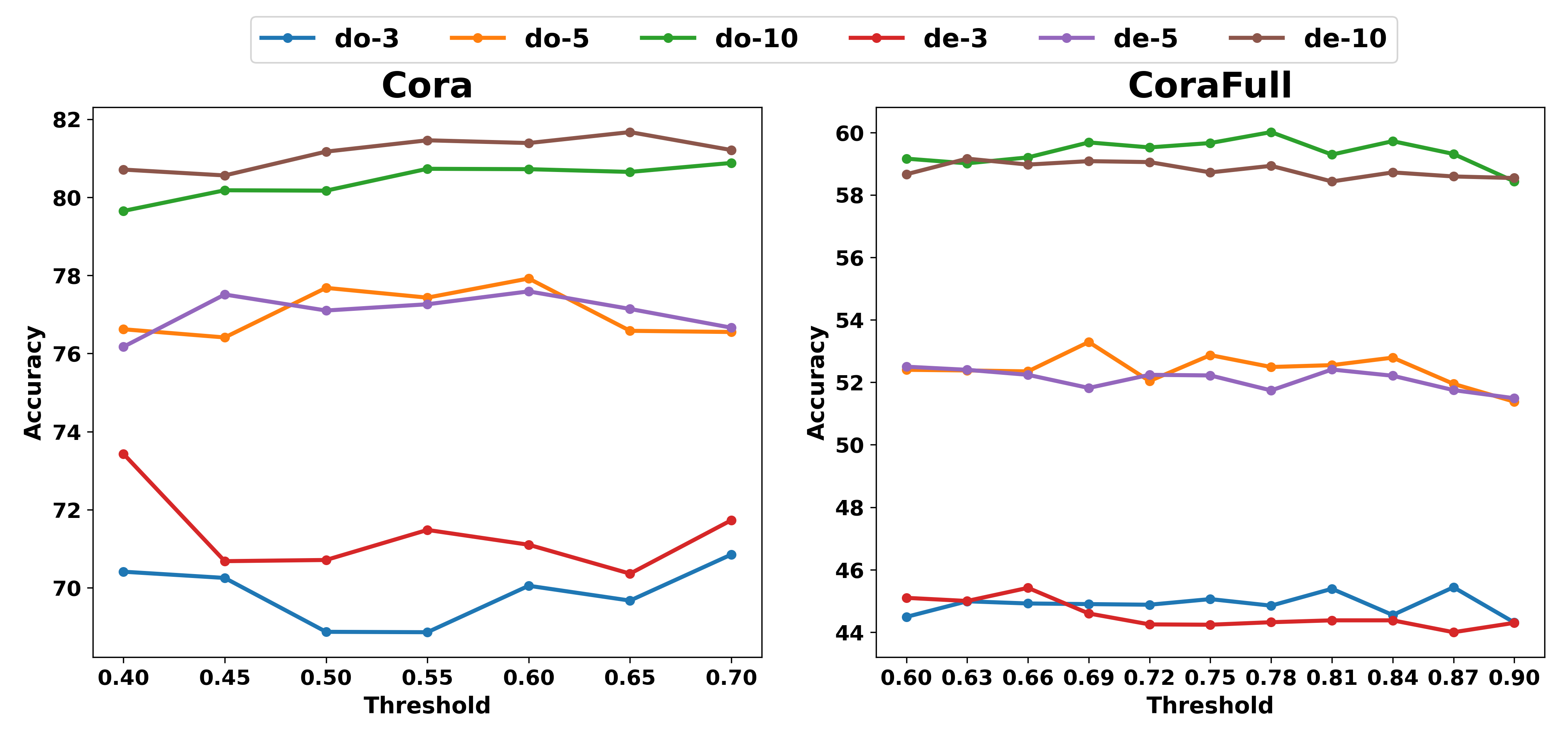}
  \caption{Impact of threshold $\tau$.}
	\label{fig:parameter_thre}
\end{figure}

\begin{figure}
	\centering
	\includegraphics[width=0.86\columnwidth, height=3cm]{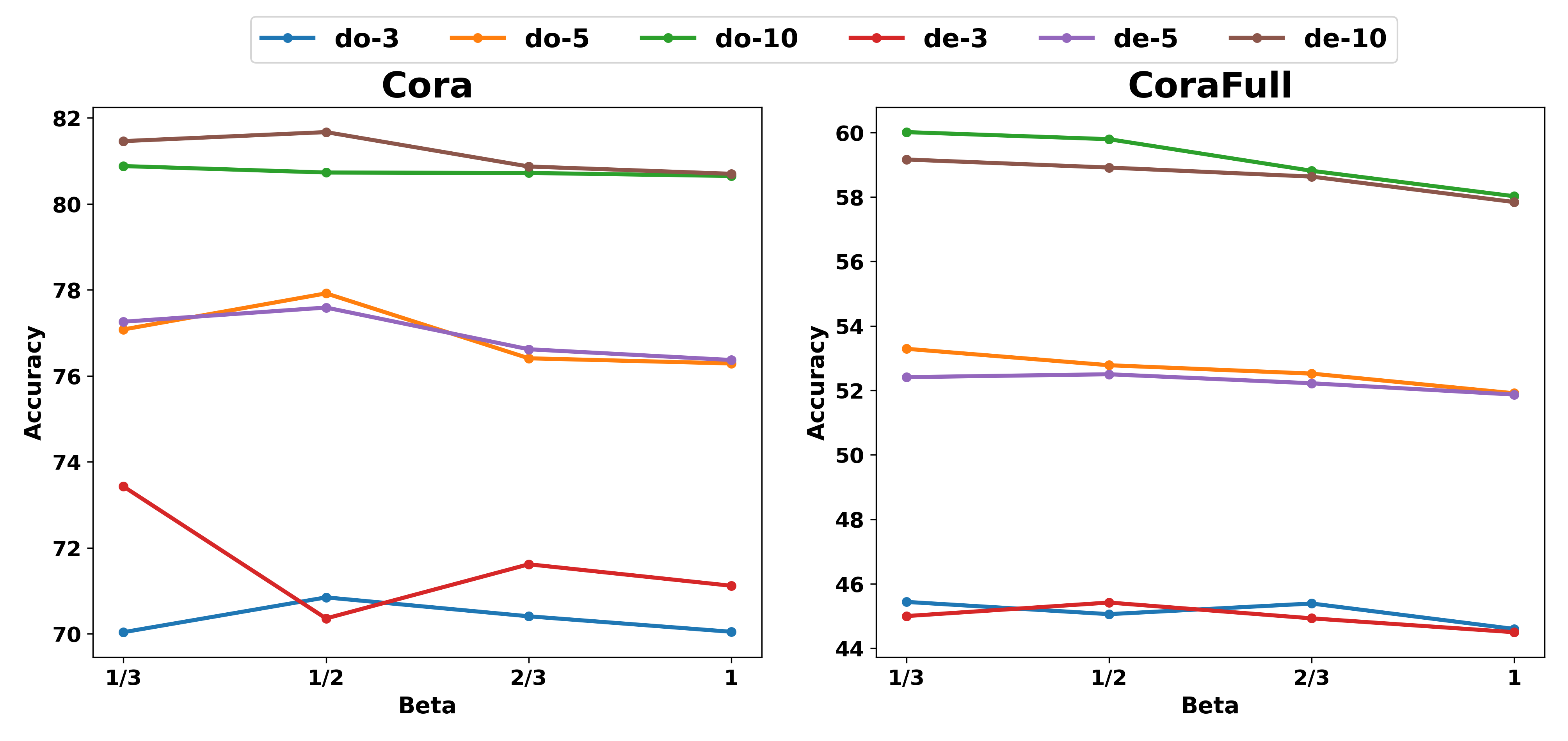}
  \caption{Impact of balance coefficient $\beta$.}
	\label{fig:parameter_beta}
\end{figure}

\begin{figure}
	\centering
	\subfigure[$stage=1$]{\label{fig:tsne_unlabeled_0}\includegraphics[width=0.32\columnwidth]{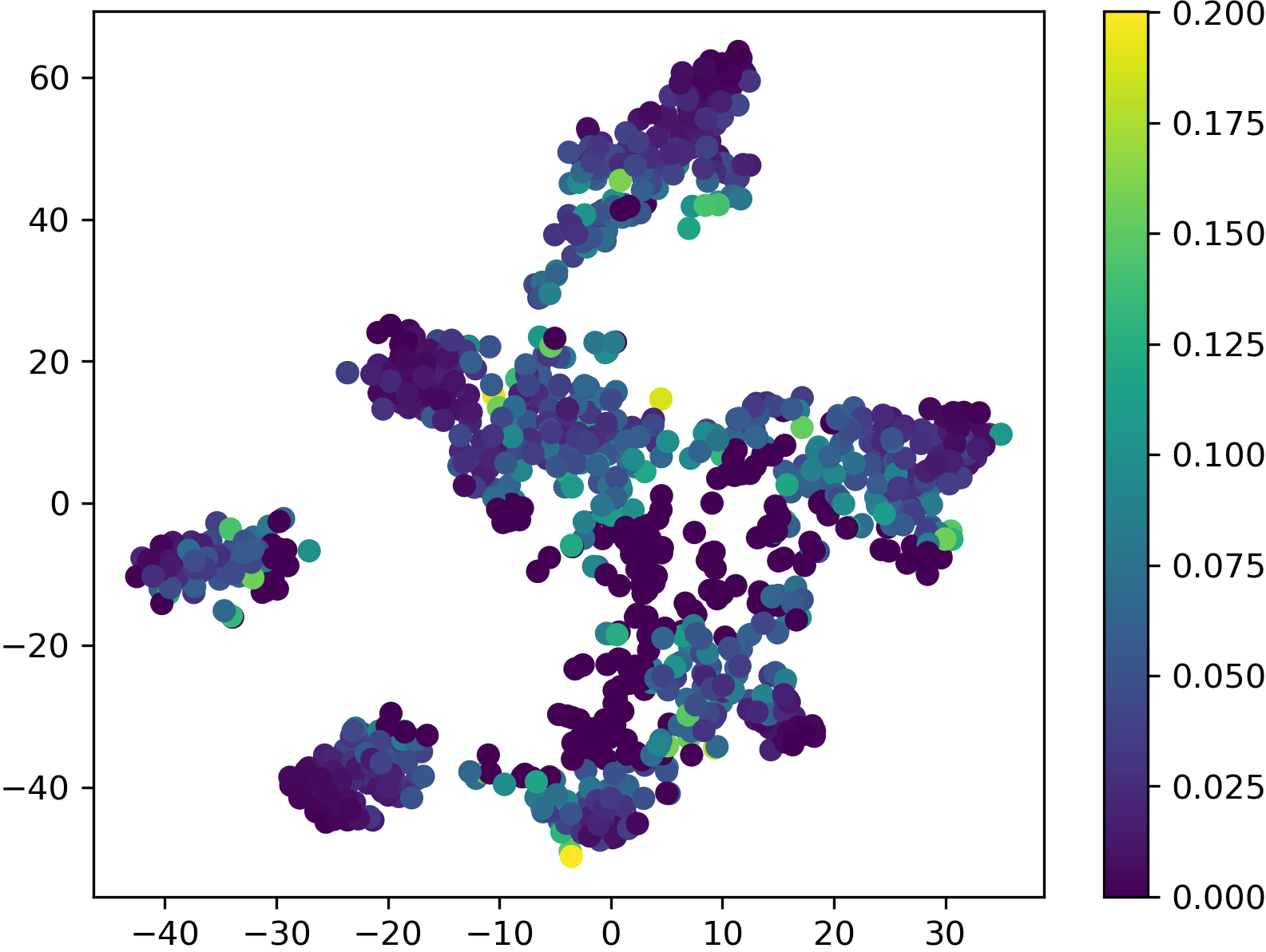}}
  \subfigure[$stage=2$]{\label{fig:tsne_unlabeled_1}\includegraphics[width=0.32\columnwidth]{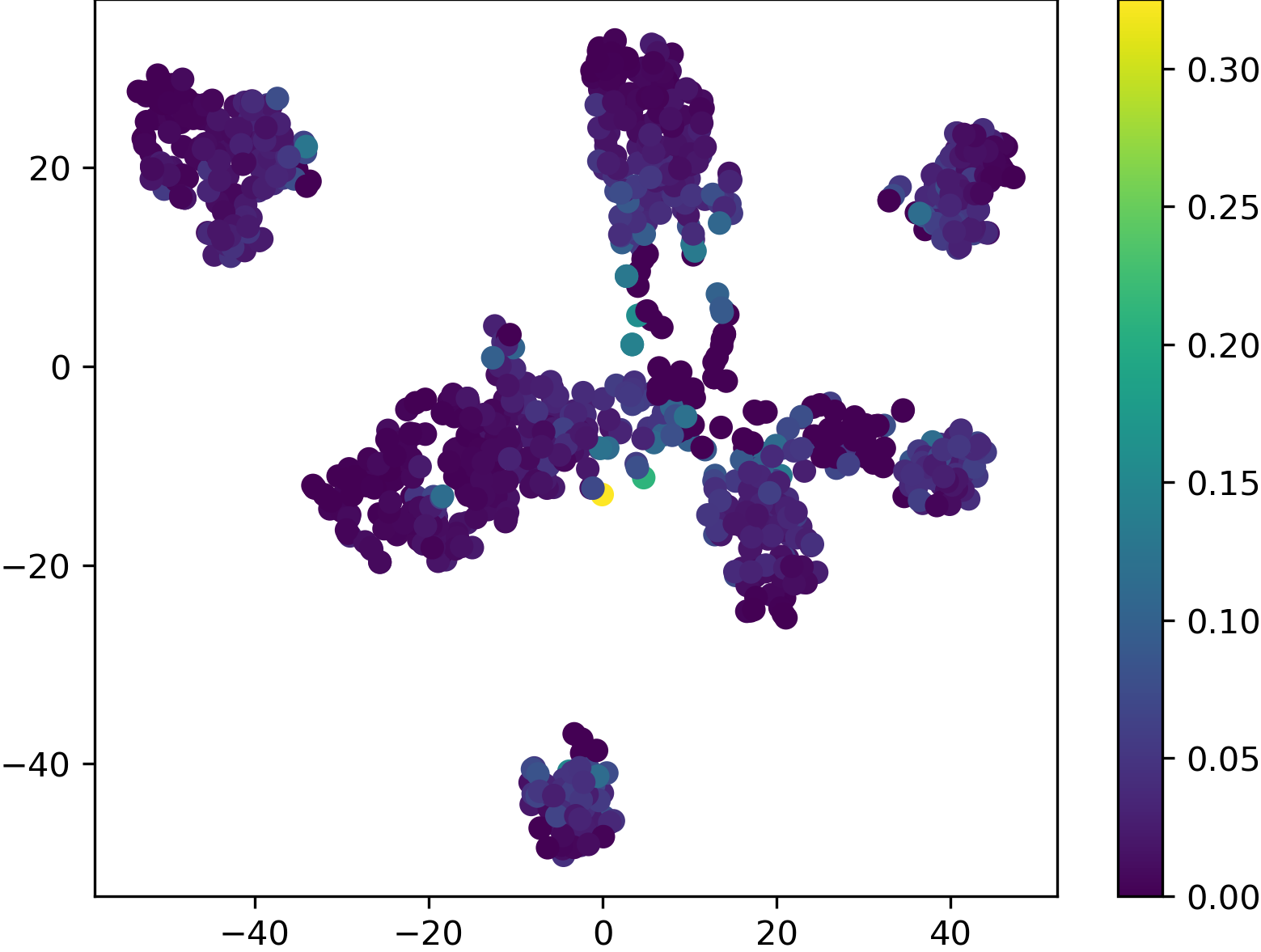}}
	\subfigure[$stage=3$]{\label{fig:tsne_unlabeled_2}\includegraphics[width=0.32\columnwidth]{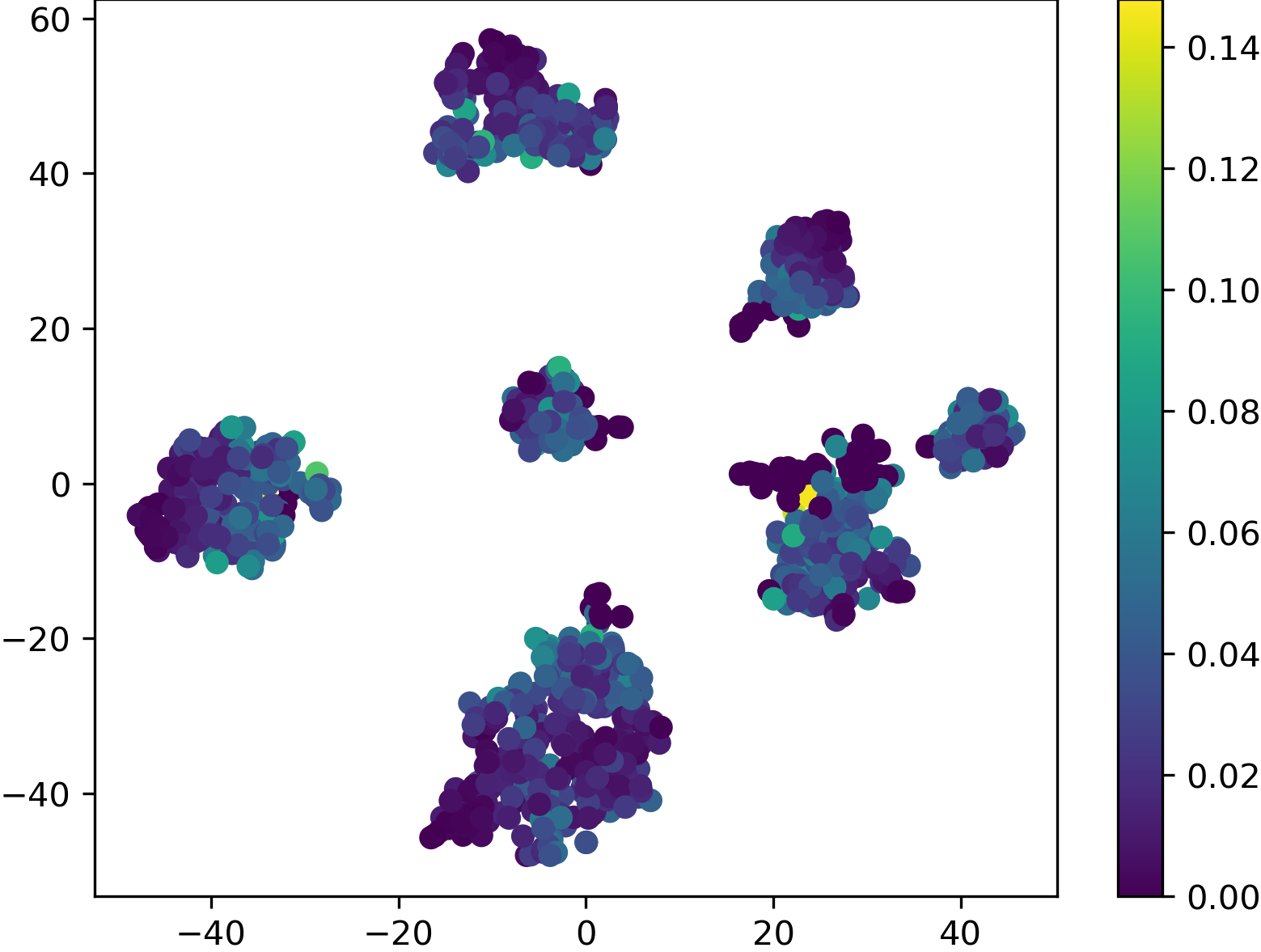}}
	
  \subfigure[$stage=1$]{\label{fig:tsne_test_0}\includegraphics[width=0.32\columnwidth]{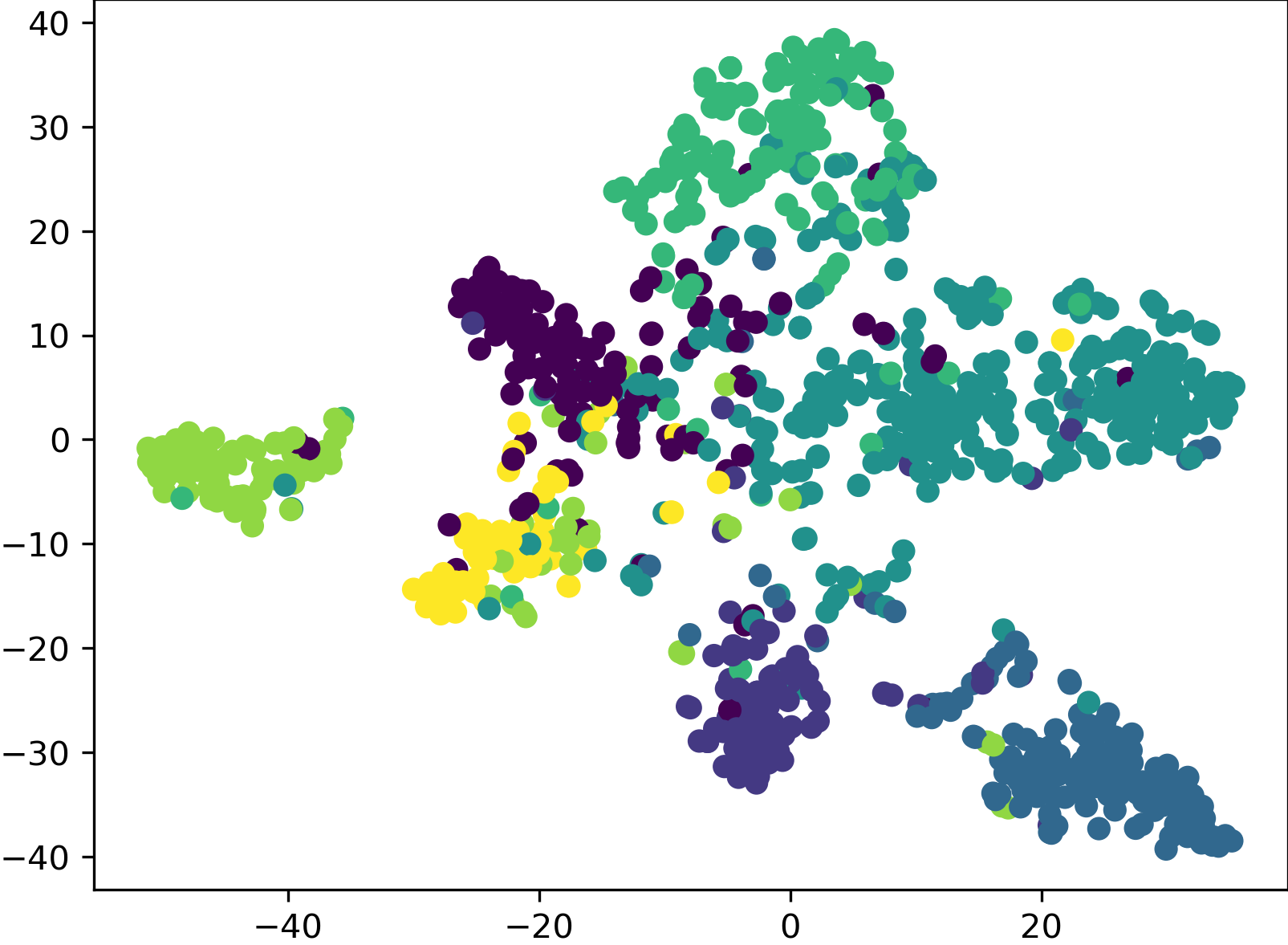}}
	\subfigure[$stage=2$]{\label{fig:tsne_test_1}\includegraphics[width=0.32\columnwidth]{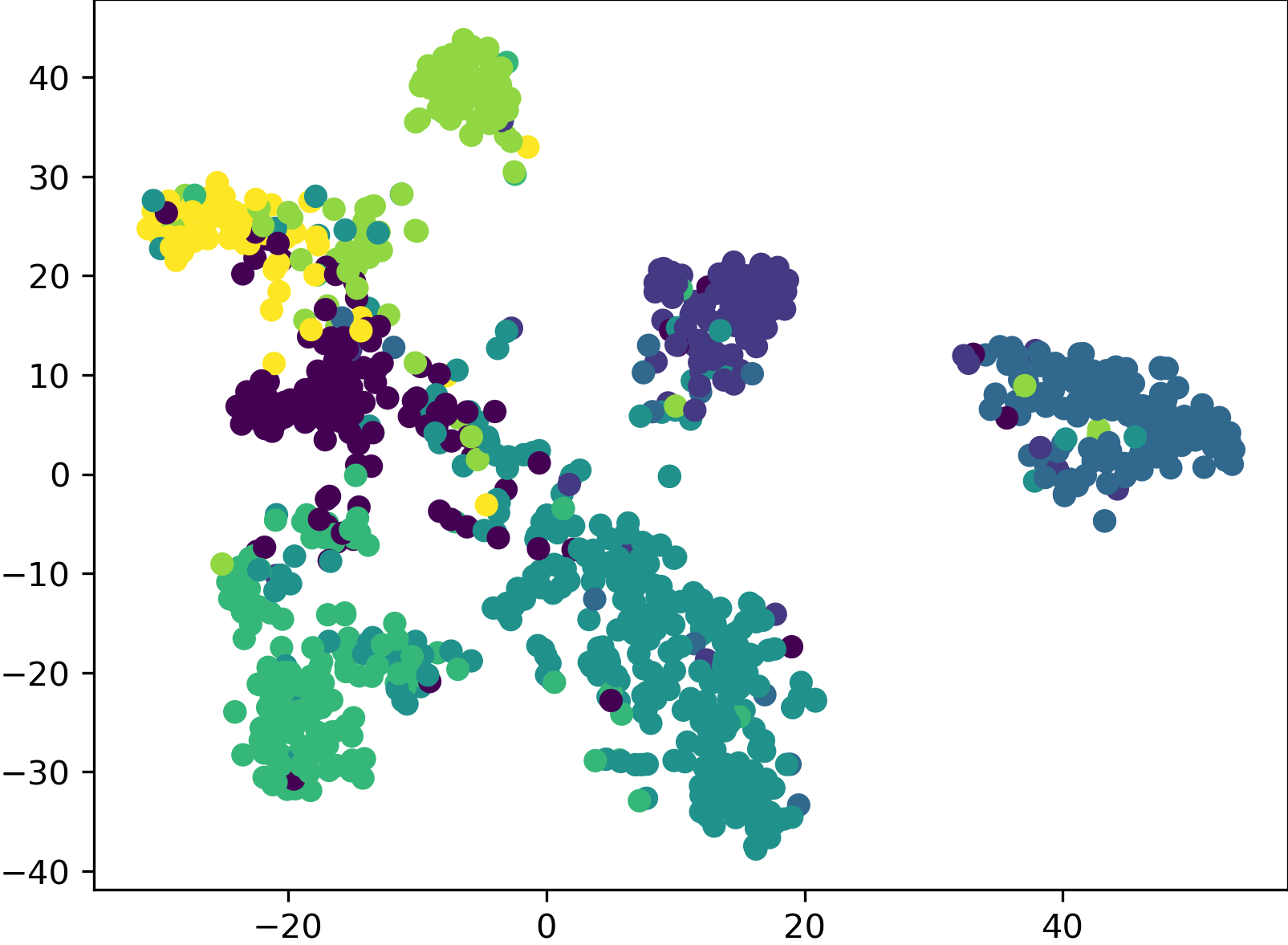}}
	\subfigure[$stage=3$]{\label{fig:tsne_test_2}\includegraphics[width=0.32\columnwidth]{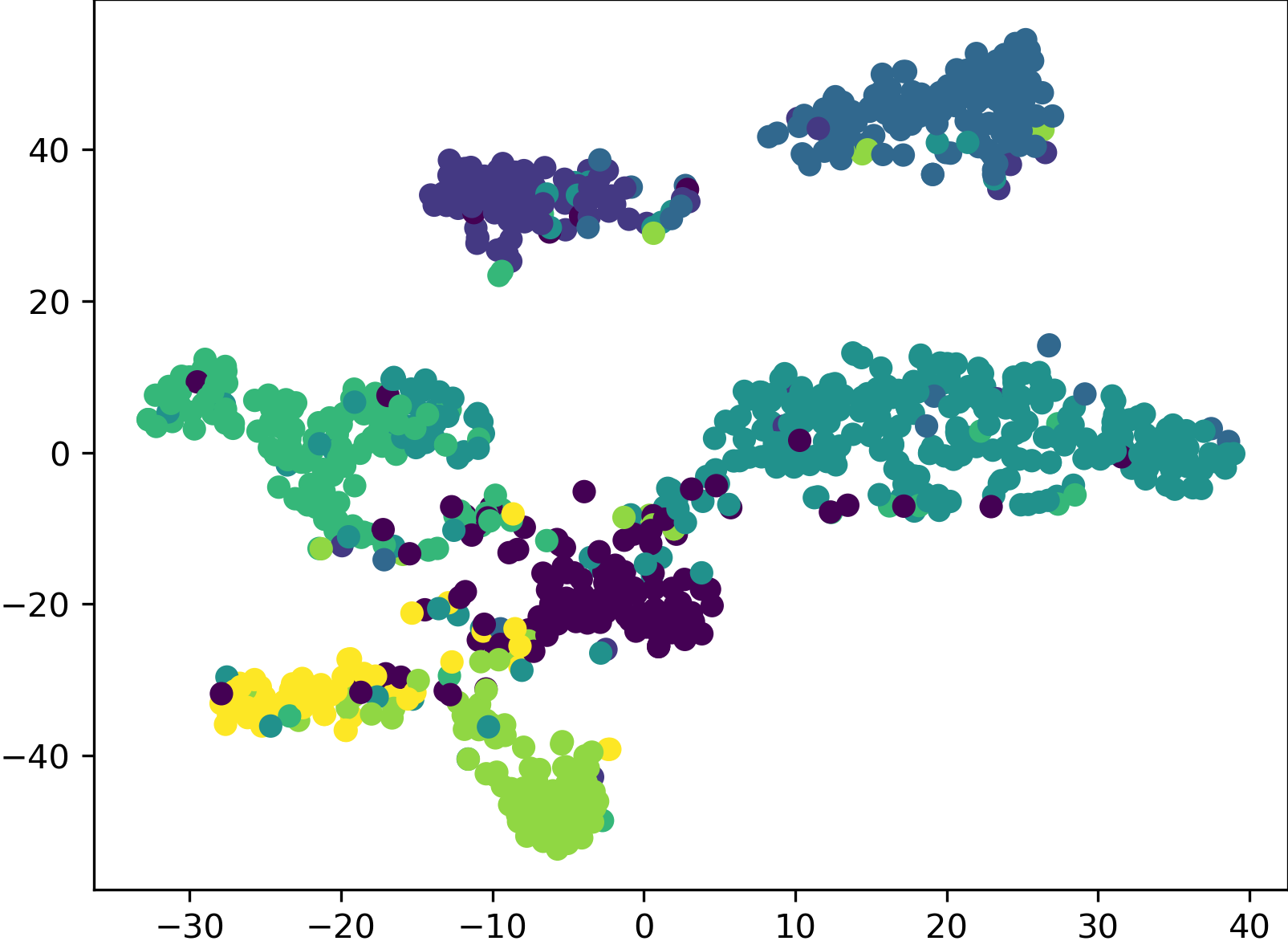}}
  \caption{Visualization of learned embeddings for unlabeled nodes ((a)$\sim$(c)) and test nodes ((d)$\sim$(f)) on Cora at different stages during self-training.}
	\label{fig:experiment_tsne}
\end{figure}

Here, we investigate into the sensitivity of two hyper-parameters (\ie threshold $\tau$ and balance coefficient $\beta$) on Cora and CoraFull datasets. Similar observations are also made on other datastes. In particular, we respectively report the performance of {\model}$_{do}$ and {\model}$_{de}$, and vary the $L/C$ in \{3, 5, 10\}. For clear notation in figures, we use ``do-3'' to denote {\model}$_{do}$ with $L/C = 3$, and the rest can be done in the same manner.

\textbf{Analysis of threshold $\tau$ in self-training} We test the impact of threshold $\tau$ in self-training, and vary it from 0.40 to 0.70 for Cora and 0.60 to 0.90 for CoraFull. The results are summarized in Fig .\ref{fig:parameter_thre}. Generally speaking, the best performance is achieved when we set a smaller $\tau$, which is consistent with our analysis above that high-confidence unlabeled nodes contribute less. 

\textbf{Analysis of balance coefficient $\beta$}
We then test the impact of the balance coefficient $\beta$ in Eq. \ref{eq:final_loss}, and vary it from 1/3 to 1. The results are shown in Fig. \ref{fig:parameter_beta}. Obviously, with the increase of $\beta$, or, in other words, with more attention paid to hard nodes, the performance shows a downward trend, further demonstrating the effectiveness of our design.

\subsubsection{Visualization}
For a more intuitive of the proposed information gain based DR-GST, we conduct the task of visualization on Cora dataset. Specifically, as shown in Fig. \ref{fig:experiment_tsne}, we visualize the output embedding of the student model at different stages in DR-GST for Cora dataset. From Fig. \ref{fig:tsne_unlabeled_0} to Fig. \ref{fig:tsne_unlabeled_2} we show the visualization of unlabeled nodes, where a lighter dot represents a node endowed with a higher weight by information gain when calculating the loss function in Eq. \ref{eq:final_loss}. Obviously, we can discover that at an earlier stage, DR-GST pays more attention to nodes close to the decision boundary which is also indistinct at this moment. With the training progress going on, the light nodes gradually vanish, implying that most of information these nodes contain has been learned, leading to a more crisp decision boundary. From Fig. \ref{fig:tsne_test_0} to Fig. \ref{fig:tsne_test_2} we show the visualization of test nodes, where different colors represent different classes. Apparently, the separability of different classes for test nodes is gradually improved, further demonstrating the effectiveness of DR-GST for optimizing the decision boundary. 

\section{Related Work}
\label{sec:related_work}

In line with the main focus of our work, we review the most related work in graph neural networks and self-training.

\textbf{Graph Neural Networks}
Recent years have seen a surge of efforts on Graph Neural Networks (GNNs) and achieved state-of-the-art performance in various tasks on graphs~\cite{zhang2020deep,wu2020comprehensive}.
Generally, current GNNs can be divided into two categories. 
The first category is spectral-based GNNs, which defines graph convolution operation in the spectral domain~\cite{Spectral-graph,ChebNet}.
The well-known GCN \cite{gcn} simplifies graph convolutions by using the 1-order approximation.
Since then, plenty of studies have sprung up. SGC \cite{sgc} further simplifies GCN by removing the nonlinearities between GCN layers. 
\cite{deeper} shows that GCNs smooth node features between neighbours.
On the comparison, the other category is spatial-based GNNs, mainly devoted to aggregating and transforming the local information from the perspective of spatial domain. GAT \cite{gat} assigns the learnt weight to each edge during aggregation. 
\cite{graphsage} proposes a permutation-invariant aggregator for message passing. 
Moreover, there are many other graph neural models, we please refer the
readers to recent surveys \cite{survey1, survey2} for a more comprehensive review.

\textbf{Self-training}
Despite the success, GNNs typically require large amounts of labeled data, which is expensive and time-consuming.
Self-training \cite{self-training} is one of the earliest strategies addressing labeled data scarcity by making better use of abundant unlabeled data, and has shown remarkable performance on various tasks \cite{pseudo-labeling, uncertainty, self-neural}.
Recently, \cite{deeper} proposes a graph-based self-training framework, demonstrating the effectiveness of self-training on graphs. Further, \cite{m3s} utilizes the DeepCluster \cite{deepcluster} to filter out low-quality pseudo labels during self-training.
CaGCN-st \cite{beconfident} argues that self-training under-performs due to generally overlooked low-confidence but high-accuracy predictions, and proposes a confidence-calibrated self-training framework.
\cite{abn} proposes to select high-quality unlabeled nodes via an adaptive pseudo labeling technique. \cite{rank-self-training} utilizes a margin prediction confidence to select unlabeled nodes, aiming at identifying the most confident labels. In summary, almost all of graph self-training methods focus on improving the quality of pseudo labels by virtue of confidence, but none of them have ever considered the capability and limitation of such selection criterion.

\section{Conclusion}
In this paper, we empirically make a thorough study for capability and limitation of current self-training methods on graphs, and surprisingly find they may be cheated by confidence and even suffer from the {\ds} issue, leading to unpromising performance. To this end, we propose a novel self-training framework {\model} which not only addresses the {\ds} issue from the view of information gain, but also is equipped with the creative loss correction strategy for improving qualities of pseudo labels. Theoretical analysis and extensive experiments well demonstrate the effectiveness of the proposed {\model}. Moreover, our study also gives an insight that confidence alone is not enough for self-training and thus motivates us an interesting direction for future work, \ie exploiting more criteria for the selection of unlabeled nodes during self-training.

\section{ACKNOWLEDGMENTS}

This work is supported in part by the National Natural Science Foundation of China (No. U20B2045, 62192784, 62172052, 61772082, 62002029, U1936104), the Fundamental Research Funds for the Central Universities 2021RC28 and CCF-Ant Group Research Fund.

\bibliographystyle{ACM-Reference-Format}
\bibliography{main}

\clearpage
\appendix
\newpage
\section{SUPPLEMENT}
In the supplement, we first provide detailed poofs of import theorems in our paper \ie Theorem~\ref{theory:loss_function}, Proposition \ref{theory_lc} and Theorem \ref{theory}. Next, more experimental details are represented for reproduction.

\subsection{Proof}
In this section, we successively show the detailed proof for Theorem~\ref{theory:loss_function}, Proposition \ref{theory_lc} and Theorem \ref{theory}.
\subsubsection{Proof of Theorem \ref{theory:loss_function}}
\label{appendix:proof_loss_function}
\begin{proof}
With our assumption that $\bar{y}_u=y_u$ for each pseudo-labeled node $v_u\in\mathcal{S}_U$, 
we first rewrite $\mathcal{L}_{pop}$ in Eq. \ref{eq:loss_pop} as:
\begin{equation}
  \label{eq:loss_pop_divide}
  \begin{aligned}
    \mathcal{L}_{pop} &= \frac{|\mathcal{S}_U|}{|\mathcal{V}_L\cup\mathcal{S}_U|}\mathbb{E}_{(v_u,y_u)\thicksim P_{pop}(\mathcal{V,Y})}l(\bar{y}_u,\mathbf{p}_u)\\
    &+\frac{|\mathcal{V}_L|}{|\mathcal{V}_L\cup\mathcal{S}_U|}\mathbb{E}_{(v_i,y_i)\thicksim P_{pop}(\mathcal{V,Y})}l(y_i,\mathbf{p}_i).
  \end{aligned}
\end{equation}
Note that 
\begin{equation}
  \mathbb{E}_{(v_u,y_u)\thicksim P_{pop}(\mathcal{V,Y})}l(\bar{y}_u,\mathbf{p}_u) = \mathbb{E}_{(v_u,y_u)\thicksim P_{st}(\mathcal{V,Y})}\frac{P_{pop}(v_u,y_u)}{P_{st}(v_u,y_u)}l(\bar{y}_u,\mathbf{p}_u),
\end{equation}
then we can rewrite Eq. \ref{eq:loss_pop_divide} as
\begin{equation}
  \label{eq:loss_pop_complex}
  \begin{aligned}
    \mathcal{L}_{pop} &= \frac{|\mathcal{S}_U|}{|\mathcal{V}_L\cup\mathcal{S}_U|}\mathbb{E}_{(v_u,y_u)\thicksim P_{st}(\mathcal{V,Y})}\frac{P_{pop}(v_u,y_u)}{P_{st}(v_u,y_u)}l(\bar{y}_u,\mathbf{p}_u)\\
    &+\frac{|\mathcal{V}_L|}{|\mathcal{V}_L\cup\mathcal{S}_U|}\mathbb{E}_{(v_i,y_i)\thicksim P_{pop}(\mathcal{V,Y})}l(y_i,\mathbf{p}_i) \\
    &=\frac{|\mathcal{S}_U|}{|\mathcal{V}_L\cup\mathcal{S}_U|}\mathbb{E}_{(v_u,y_u)\thicksim P_{st}(\mathcal{V,Y})}\gamma_u l(\bar{y}_u,\mathbf{p}_u)\\
    &+\frac{|\mathcal{V}_L|}{|\mathcal{V}_L\cup\mathcal{S}_U|}\mathbb{E}_{(v_i,y_i)\thicksim P_{pop}(\mathcal{V,Y})}l(y_i,\mathbf{p}_i),
  \end{aligned}
\end{equation}
where $\gamma_u$ can be regarded as a weight of the loss function for each pseudo-labeled node $v_u$. 

Finally, recalling the loss function under the {\ds} case in Eq. \ref{eq:loss_st}, i.e., 
\begin{equation}
    \begin{aligned}
        \mathcal{L}_{st} &= \frac{|\mathcal{V}_L|}{|\mathcal{V}_L\cup\mathcal{S}_U|} \mathbb{E}_{(v_i,y_i)\thicksim P_{pop}(\mathcal{V,Y})}l(y_i,\mathbf{p}_i) \\
        &+ \frac{|\mathcal{S}_U|}{|\mathcal{V}_L\cup\mathcal{S}_U|}\mathbb{E}_{(v_u,y_u)\thicksim P_{st}(\mathcal{V,Y})}l(\bar{y}_u,\mathbf{p}_u),
    \end{aligned}
\end{equation}
we can find that it is definitely equal to that in Eq. \ref{eq:loss_pop} 
with an additional weight coefficient. 
In other words, we can recover the population distribution as long as we weight each pseudo-labeled node with a proper coefficient in $\mathcal{L}_{st}$. 
\end{proof}

\subsubsection{Proof of Proposition \ref{theory_lc}}
\label{appendix:proof_2}
\begin{proof}
    Without loss of generality, we respectively prove the equality of ${\theta}^*$ and $\bar{\theta}$ under MSE loss and CE loss.
    
    \textbf{MSE loss.} Under the MSE loss, with our non-zero assumption for $\mathbf{T}$, the following equation holds true: 
    \begin{equation}
      \begin{aligned}
        {\theta^*} &= \arg\min_{{\theta^*}\in\Theta}\sum_u||f_{\theta^*}(\mathbf{x}_u,\mathbf{A})-\mathbf{y}_u||^2\\
        &=\arg\min_{{\theta^*}\in\Theta}\sum_u||\mathbf{T}f_{\theta^*}(\mathbf{x}_u,\mathbf{A})-\mathbf{T}\mathbf{y}_u||^2\\
        &=\arg\min_{\bar{\theta}\in\Theta}\sum_u||f_{\bar{\theta}}(\mathbf{x}_u,\mathbf{A})-\bar{\mathbf{y}}_u||^2=\bar{\theta}, 
      \end{aligned}
    \end{equation}
    where $\mathbf{y}_u$ is a one-hot vector expanded from $y_u$. The proof is concluded for MSE loss.
    
    \textbf{CE loss.}
    Under CE loss, we prove the equality of $\bar{\theta}$ and $\theta^*$ from the perspective of gradient descent. Specifically, if for each node $v_u$, the gradient of $f_{\bar{\theta}}(\mathbf{x}_u,\mathbf{A})$ w.r.t.  $\bar{\theta}$ is equal to that of $f_{\theta ^*}(\mathbf{x}_u,\mathbf{A})$ w.r.t. $\theta^*$, then optimizing a model $f_{\bar{\theta}}$ using gradient descent will definitely leads to our desired model $f_{\theta^*}$, that is to say, $\bar{\theta}=\theta^*$.
    
    Specifically, for each node $v_u$, we first rewrite the CE loss as follows:
    \begin{equation}
        l(\mathbf{y}_u,\mathbf{p}_u)=\mathbf{y}_u^\mathsf{T}\log f_\theta(\mathbf{x}_u,\mathbf{A}).
    \end{equation}
    Then the difference $d$ of gradient between $\bar{\theta}$ and $\theta^*$ can be written as:
    \begin{equation}
    \label{eq:appendix_diff}
        d=||\nabla_\theta {\bar{\mathbf{y}}_u}^\mathsf{T}\log f_{\bar{\theta}}(\mathbf{x}_u,\mathbf{A})-\nabla_\theta \mathbf{y}_u^\mathsf{T}\log f_{\theta^*}(\mathbf{x}_u,\mathbf{A})||
    \end{equation}
    Considering our assumption that $f_{\bar{\theta}}(\mathbf{x}_u,\mathbf{A})=\mathbf{T}f_{\theta^ *}(\mathbf{x}_u,\mathbf{A})$, Eq. \ref{eq:appendix_diff} becomes:
    \begin{equation}
    \label{eq:appendix_diff_re}
        d=||\nabla_\theta ({\mathbf{T}\mathbf{y}_u})^\mathsf{T}\log (\mathbf{T}f_{\theta^*}(\mathbf{x}_u,\mathbf{A}))-\nabla_\theta \mathbf{y}_u^\mathsf{T}\log f_{\theta^*}(\mathbf{x}_u,\mathbf{A})||
    \end{equation}
    According to the chain rule, we have:
    \begin{equation}
        d=||\nabla_\theta f_{\theta^*}(\mathbf{x}_u,\mathbf{A})\cdot (\mathbf{T}^\mathsf{T}(\mathbf{T}\mathbf{y}_u\oslash\mathbf{T}f_{\theta^*}(\mathbf{x}_u,\mathbf{A}))-\mathbf{y}_u\oslash f_{\theta^*}(\mathbf{x}_u,\mathbf{A}))||,
    \end{equation}
    where $\oslash$ represents the element-wise division operation.
    
    Obviously, if $\mathbf{T}$ is a permutation matrix, the difference $d$ of gradient is zero. The proof is concluded for CE loss.
\end{proof}

\subsubsection{Proof of Theorem \ref{theory}}
\label{appendix:proof_1}

To prove Theorem \ref{theory}, we need to borrow a corollary from \cite{abn}, which illustrates the impact of incorrect pseudo labels on self-training without {\ds}.
\begin{corollary}
\label{corallary}
    Assuming that the augmented dataset follows the population distribution $P_{pop}$ and $||\nabla_\theta l||\leq\Psi$ for any gradient $\nabla_\theta\mathcal{L}$, the following bound between $\nabla_\theta\mathcal{L}_{pop}$ and $\nabla_\theta\mathcal{L}_{st}$ holds:
    \begin{equation}
    \begin{aligned}
        |\nabla_\theta \mathcal{L}_{pop} - \nabla_\theta \mathcal{L}_{st}| &\leqslant \frac{|\mathcal{S}_U|}{|\mathcal{V}_L\cup\mathcal{S}_U|}2\Psi||P_{(v_u,y_u)\thicksim P_{pop}(\mathcal{V,Y})}(\bar y_u\neq y_u)||.
    \end{aligned}
    \end{equation}
\end{corollary}
Now, we prove Theorem \ref{theory}.
\begin{proof}
We first calculate the difference between $\nabla_\theta\mathcal{L}_{pop}$ and $\nabla_\theta\mathcal{L}_{st}$ as follows:
\begin{equation}
\label{eq:appendix_proof_theory1}
    \begin{aligned}
        ||\nabla_\theta\mathcal{L}_{pop}-\nabla_\theta\mathcal{L}_{st}||
        &= \frac{|\mathcal{S}_U|}{|\mathcal{V}_L\bigcup\mathcal{S}_U|}||\mathbb{E}_{(v_u,y_u)\sim P_{pop}(\mathcal{V,Y})}\nabla_{\theta}l(y_u,\mathbf{p}_u)\\
        &-\mathbb{E}_{(v_u,y_u)\sim P_{st}(\mathcal{V,Y})}\nabla_\theta l(\bar{y}_u,\mathbf{p}_u)||.
    \end{aligned}
\end{equation}
Adding and subtracting a same term $\mathbb{E}_{(v_u,y_u)\sim P_{st}(\mathcal{V,Y})}\nabla_\theta l(\bar{y}_u,\mathbf{p}_u)$, and abbreviating $\frac{|\mathcal{S}_U|}{|\mathcal{V}_L\bigcup\mathcal{S}_U|}$ as $\eta$, Eq. \ref{eq:appendix_proof_theory1} can be written as:
\begin{equation}
    \begin{aligned}
        &||\nabla_\theta\mathcal{L}_{pop}-\nabla_\theta\mathcal{L}_{st}||=\\
        &\eta||\mathbb{E}_{(v_u,y_u)\sim P_{pop}(\mathcal{V,Y})}\nabla_\theta l(y_u,\mathbf{p}_u)-\mathbb{E}_{(v_u,y_u)\sim P_{pop}(\mathcal{V,Y})}\nabla_\theta l(\bar{y}_u,\mathbf{p}_u)\\
        &+\mathbb{E}_{(v_u,y_u)\sim P_{pop}(\mathcal{V,Y})}\nabla_\theta l(\bar{y}_u,\mathbf{p}_u)-\mathbb{E}_{(v_u,y_u)\sim P_{st}(\mathcal{V,Y})}\nabla_\theta l(\bar{y}_u,\mathbf{p}_u)||.
    \end{aligned}
\end{equation}
According to the triangle property of the norm, the following inequality is satisfied:
\begin{equation}
    \begin{aligned}
        &||\nabla_\theta\mathcal{L}_{pop}-\nabla_\theta\mathcal{L}_{st}||\leq\\
        &\eta (||\mathbb{E}_{(v_u,y_u)\sim P_{pop}(\mathcal{V,Y})}\nabla_\theta l(y_u,\mathbf{p}_u)-\mathbb{E}_{(v_u,y_u)\sim P_{pop}(\mathcal{V,Y})}\nabla_\theta l(\bar{y}_u,\mathbf{p}_u)||\\
        &+||\mathbb{E}_{(v_u,y_u)\sim P_{pop}(\mathcal{V,Y})}\nabla_\theta l(\bar{y}_u,\mathbf{p}_u)-\mathbb{E}_{(v_u,y_u)\sim P_{st}(\mathcal{V,Y})}\nabla_\theta l(\bar{y}_u,\mathbf{p}_u)||).
    \end{aligned}
\end{equation}
Recalling Corollary \ref{corallary}, we know that the first term on the right hand side satisfies:
\begin{equation}
\label{eq:appendix_proof_first}
    \begin{aligned}
        &||\mathbb{E}_{(v_u,y_u)\sim P_{pop}(\mathcal{V,Y})}\nabla_\theta l(y_u,\mathbf{p}_u)-\mathbb{E}_{(v_u,y_u)\sim P_{pop}(\mathcal{V,Y})}\nabla_\theta l(\bar{y}_u,\mathbf{p}_u)||\\
        &\leq 2\Psi||P_{(v_u,y_u)\thicksim P_{pop}(\mathcal{V,Y})}(\bar y_u\neq y_u)||.
    \end{aligned}
\end{equation}
And for the second term, we have:
\begin{equation}
\label{eq:appendix_proof_second}
    \begin{aligned}
        &\mathbb{E}_{(v_u,y_u)\sim P_{pop}(\mathcal{V,Y})}\nabla_\theta l(\bar{y}_u,\mathbf{p}_u)-\mathbb{E}_{(v_u,y_u)\sim P_{st}(\mathcal{V,Y})}\nabla_\theta l(\bar{y}_u,\mathbf{p}_u)||\\
        &=\int_{-\infty}^{+\infty}\int_{-\infty}^{+\infty}\nabla_\theta l(\bar{y}_u,\mathbf{p}_u)d(P_{pop}(\mathcal{V,Y})-P_{st}(\mathcal{V,Y}))\\
        &\leq \Psi\cdot||P_{pop}(\mathcal{V,Y})-P_{st}(\mathcal{V,Y})||,
    \end{aligned}
\end{equation}
where the inequality is from our assumption that $||\nabla_\theta l||\leq\Psi$.

Combining Eq. \ref{eq:appendix_proof_first} with Eq. \ref{eq:appendix_proof_second}, we have:
\begin{equation}
    \begin{aligned}
        ||\nabla_\theta \mathcal{L}_{pop} - \nabla_\theta \mathcal{L}_{st}|| &\leqslant \frac{|\mathcal{S}_U|}{|\mathcal{V}_L\cup\mathcal{S}_U|}\Psi(2||P_{(v_u,y_u)\thicksim P_{pop}(\mathcal{V,Y})}(\bar y_u\neq y_u)\\
        &+||P_{st}(\mathcal{V,Y})-P_{pop}(\mathcal{V,Y})||).
    \end{aligned}
\end{equation}
The proof is concluded. 
\end{proof}

\subsection{Time Complexity Analysis}
\label{appendix:time}
We first analyze the time complexity of a general self-training framework. Assuming training an epoch takes $O(M)$ time, given epochs $E$, its time complexity in each stage is $O(EM)$. 
DR-GST is innovated in information gain and loss correction, which respectively takes $O(TM)$ and $O(Ec^2)$ time in each stage, where $T$ and $c$ are the numbers of sampling for variational inference and class. Moreover, considering that we train a student model twice in each stage, the total time complexity is $O((2E+T)M+Ec^2)$.
In fact, $T$ and $O(Ec^2)$ are always far less than $E$ and $O(EM)$. Consequently, the time complexity of DR-GST is approximately twice that of the general self-training framework.

\subsection{More Experimental Details}
\label{appendix:details}

\subsubsection{Details of datasets}
\label{appendix:dataset}
\begin{table}[]
  \caption{The statistics of datasets}
  \label{tab:dataset}
  \setlength{\tabcolsep}{0.9mm}{
    \begin{tabular}{@{}ccccccc@{}}
      \toprule
      \textbf{Dataset}  & \textbf{Nodes} & \textbf{Edges} & \textbf{Classes} & \textbf{Features} & \textbf{Validation} & \textbf{Test} \\ \midrule
      \textbf{Cora}     & 2708           & 5429           & 7                & 1433              & 500                 & 1000          \\
      \textbf{Citeseer} & 3327           & 4732           & 6                & 3703              & 500                 & 1000          \\
      \textbf{Pubmed}   & 19717          & 44338          & 3                & 500               & 500                 & 1000          \\
      \textbf{CoraFull} & 19793          & 65311          & 70               & 8710              & 500                 & 1000          \\
      \textbf{Flickr}   & 7575           & 239738         & 9                & 12047             & 500                 & 1000          \\ \bottomrule
      \end{tabular}}
\end{table}
We adopt five widely used benchmark datasets from citation networks~\cite{cora,corafull} (\ie Cora, Citeseer, Pubmed and CoraFull) and social network~\cite{flickr} (\ie Flickr) for evaluation. For the citation networks, nodes represent papers, edges are the citation relationship between papers, node features
are comprised of bag-of-words vector of the papers and labels represent the fields of papers. 
And for the social network, nodes in Flickr represent users of the Flickr website, edges are their relationships 
induced by their photo-sharing records and labels represent users' interest groups. 
For all the datasets, We choose 500 nodes for validation, 1000 nodes for test. The details of these datasets are summarized in Table \ref{tab:dataset}.
Our data are public and do not contain personally identifiable information and offensive content.
The address of our data is \url{https://docs.dgl.ai/en/latest/api/python/dgl.data.html#node-prediction-datasets} 
and the license is Apache License 2.0.

\subsubsection{Implementation}
\label{appendix:implementation}
We supplement the implementation details of DR-GST and all the baselines here.

For fair comparison, we utilize the standard GCN with 2 layers as the backbone for all graph self-training framework. We optimize models via Adam with learning rate of 0.01 and early stopping with a window size of 200.
In paticular, we set L2 regularization with $\lambda_{r}=5e-4$ 
for Cora, Citeseer, Pubmed, CoraFull and $\lambda_{r}=5e-5$ for Flickr. 
We set ReLU as the activation function and apply a dropout rate of 0.5 to prevent over-fitting. 
As for the MC-dropout and MC-dropedge, we set the number of sampling $T=100$.
Moreover, we apply grid search for other important hyper-parameters. Specifically, the drop rate of MC-dropout and MC-dropedge is chosen from $\{0.1,0.2,\cdots,0.5\}$, 
the balance coefficient $\beta$ for information gain in Eq. \ref{eq:final_loss} is searched in $\{4/3,1,2/3,1/2,1/3,1/4\}$ and the threshold $\tau$ 
is tuned amongst $\{0.4,0.45,\cdots,0.75\}$ for Cora, Citeseer, $\{0.6,0.65,\cdots,0.9\}$ for Pubmed, CoraFull and $\{0.75,0.78,\cdots,0.96\}$ for Flickr.

We adopt the implementation of GCN, GAT and APPNP from DGL\footnote{https://www.dgl.ai/}, and the implementations of STs~\footnote{https://github.com/Davidham3/deeper\_insights\_into\_GCNs} and ABN~\footnote{https://anonymous.4open.science/r/e7aca211-0d8d-4564-8f3f-0ef24b01941e/} are publicly provided by their authors. Considering that the implementation of M3S is not available, we re-implement it referring to the original paper~\cite{m3s}. For all baselines, we perform grid search for important hyper-parameters (\ie $\tau$) to obtain optimal results.

\subsubsection{Experimental Environment}
\label{appendix:environment}
In this section we summarize the hardware and software environment in our experiments.

We utilize a linux machine powered by an Intel(R) Xeon(R) CPU E5-2682 v4 @ 2.50GHz CPU and 4 Tesla P100-PCIE-16GB as well as 4 GeForce RTX 3090 GPU cards.

The operating system is Linux version 3.10.0-693.el7.x86\_64. We realize our code with Python 3.8.8 as well as some other python packages as follows: PyTorch 1.8.1, DGL 0.6.0 (cuda 10.1), NetworkX 2.5.









\end{document}